\documentclass{article}

\usepackage[main, final]{neurips_2026}

\usepackage[utf8]{inputenc} 
\usepackage[T1]{fontenc}    
\usepackage{hyperref}       
\usepackage{url}            
\usepackage{booktabs}       
\usepackage{amsfonts}       
\usepackage{nicefrac}       
\usepackage{microtype}      
\usepackage{xcolor}         

\usepackage{amsmath,amssymb,amsthm}
\usepackage{wrapfig}
\usepackage{caption}
\usepackage{algorithm}
\usepackage{algpseudocode}
\usepackage{graphicx}
\usepackage{enumitem}
\usepackage{bm}

\newcommand{\R}{\mathbb{R}}
\newcommand{\E}{\mathbb{E}}
\newcommand{\prob}{\mathbb{P}}

\newcommand{\1}{\bm{1}}

\newtheorem{theorem}{Theorem}[section]
\newtheorem{proposition}[theorem]{Proposition}
\newtheorem{lemma}[theorem]{Lemma}

\newtheorem{assumption}[theorem]{Assumption}

\newtheorem{example}{Example}[section]

\newcommand{\rmd}{\mathop{}\!\mathrm{d}}
\title{Metropolis-Adjusted Diffusion Models}

\author{%
  Kevin H. Lam \thanks{Authors contributed equally to this work.}\\
  Department of Statistics\\
  University of Oxford\\
  \And
  Tyler Farghly \footnotemark[1]\\
  Inria, CNRS, ENS\\
  PSL Research University \\
  \And
  Christopher Williams \footnotemark[1]\\
  Department of Statistics\\
  University of Oxford\\
  \AND
  Jun Yang \\
  Dept. of Math. Sciences \\
  University of Copenhagen \\
  \And
  Yee Whye Teh\\
  Department of Statistics\\
  University of Oxford\\
  \And
  Arnaud Doucet\\
  Department of Statistics\\
  University of Oxford\\
  \AND
    {\normalsize \texttt{$\{$lam,williams,y.w.teh,doucet$\}$@stats.ox.ac.uk},
    }\\
    {\normalsize \texttt{tyler.farghly@inria.fr}, \texttt{jy@math.ku.dk}}
}

\begin{document}

\maketitle

\begin{abstract}
Sampling from score-based diffusion models incurs bias due to both time discretisation and the approximation of the score function. 
A common strategy for reducing this bias is to apply corrector steps based on the unadjusted Langevin algorithm (ULA) at each noise level within a predictor-corrector framework. 
However, ULA is itself a \emph{biased} sampler, as it discretises a continuous diffusion process. 
In this work, we consider \emph{adjusted} Langevin correctors that employ Metropolis--Hastings (MH) or Barker's accept-reject steps to correct for this bias.
Since the target density ratio typically required by MH-based algorithms is unavailable, we propose methods that instead utilise the score function to compute the correct acceptance probability. 
We introduce the first exact method for adjusting Langevin corrections in diffusion models, based on a two-coin Bernoulli factory algorithm. 
We also propose an efficient approximation based on Simpson's rule that achieves accuracy of order $5/2$ in the step size at near-zero marginal cost.
We demonstrate that these procedures improve sample quality on both synthetic and image datasets, yielding consistent gains in Fr\'echet Inception Distance (FID) on the latter.
\end{abstract}

\section{Introduction}
Diffusion models have become a key class of generative models, with state-of-the-art performance in domains such as image, audio, and protein design. They rely on a simple but powerful idea. Instead of modelling the data density directly, one learns its score along a noising path and then simulates a reverse-time diffusion process to transform noise into data \citep{sohl-dickstein2015dpm,song2019smld,ho2020ddpm,song2021sgm}. In practice, this reverse diffusion is implemented using numerical integrators \citep{song2021sgm, karras2022edm}. Consequently, even if the learned score were exact, the generated distribution would still be impacted by time-discretisation error. With a learned score, this error is compounded by approximation error in the score network.

This time-discretisation issue is common in Markov chain Monte Carlo (MCMC). To sample from a target distribution, the unadjusted Langevin algorithm (ULA), a discretisation of the Langevin diffusion, uses gradients of a log-density to explore the state space \citep{durmus2017nonasymptotic}, but this does not leave the target distribution invariant. The standard approach to correct for this discretisation bias is to add an accept--reject step, yielding the Metropolis-adjusted Langevin algorithm (MALA) \citep{besag1994comments,roberts1996exponential}. The analogy with diffusion sampling is immediate. Predictor--Corrector samplers \citep{song2021sgm} already use Langevin-type corrector steps at fixed noise levels: after a predictor step moves the sample between successive marginals, a corrector step refines the sample with respect to the current marginal. The usual corrector is again ULA. Thus, the ``corrector'' in the Predictor--Corrector scheme is itself biased. It may reduce the error introduced by the predictor, but it introduces its own discretisation error.

This motivates a natural question: can we replace ULA correctors in diffusion models by Metropolis-adjusted correctors? A direct application of MALA is not possible as it requires evaluating the density ratio $p_t(\widetilde{\mathbf{x}}) / p_t(\mathbf{x})$ between the proposal $\widetilde{\mathbf{x}}$ and the current state $\mathbf{x}$ for the intractable target marginal $p_t$ at noise level $t$, whereas diffusion models only provide a score estimate $s_\phi(\mathbf{x}, t) \approx \nabla_{\mathbf{x}} \log p_t(\mathbf{x})$.

\begin{figure}
    \centering
    \includegraphics[width=0.95\linewidth]{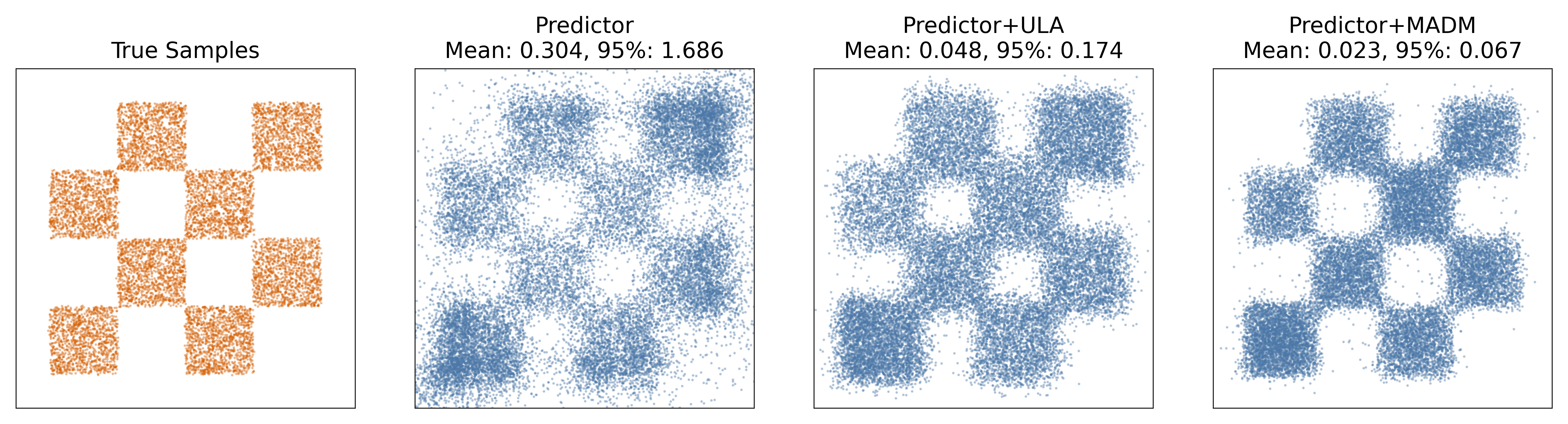}
    \caption{DDM sampling using a probability flow ODE predictor, ULA correction, and MADM correction. Mean distance to the true samples and the 95\% sample containment distance are shown. A coarse diffusion schedule and large ULA stepsize are chosen to show the discretisation bias, the same schedule and stepsize are used in MADM to show the bias correction from this method.
    \vspace{-1em}}
    \label{fig:MADM-checkerboard}
\end{figure}

To address this problem, we propose \emph{Metropolis-Adjusted Diffusion Models} (MADM), a simple modification of the Predictor--Corrector framework where the usual ULA corrector is replaced by an accept--reject corrected Langevin step. The key point is that, although the density ratio is unavailable, its logarithm can be represented as a line integral of the score. 
This result shows that the score can be used not only to construct Langevin proposals, but also to estimate the acceptance probability that corrects their discretisation bias.

We develop two approaches to exploit this representation. First, we give an exact Barker adjusted corrector based on a two-coin Bernoulli factory \citep{Goncalves2017-lh}. Under the assumption that $s(\cdot,t)=\nabla \log p_t(\cdot)$ and that a suitable bound on the estimator of the line integral is available, this procedure produces an accept--reject mechanism with the Barker acceptance probability without ever evaluating $p_t$. This gives, to our knowledge, the first exact Metropolis-style Langevin corrector 
for any smooth target distribution only known through its score. Second, we introduce a practical approximation of the density ratio based on Simpson's rule. The resulting MH correction is no longer exact, but it is very cheap. The resulting approximation has high-order accuracy in the Langevin step size, is easy to implement, and only requires access to a pre-trained score function. 

Empirically, Metropolis adjustments address a failure mode of unadjusted correctors on toy examples; see Figure \ref{fig:MADM-checkerboard}. On image datasets, we combined a predictor with an MADM corrector based on Simpson's rule using pre-trained models from \citet{karras2022edm}. While gains are modest in some unconditional settings, we found that MADM yields consistent and systematic improvements in FID.

Our contributions are as follows:
\begin{itemize}
    \item We introduce MADM, a Metropolis-adjusted corrector framework for diffusion models that replaces biased ULA corrector steps.
    \item We propose an exact method based on the Barker acceptance rule and a two-coin Bernoulli factory, giving an adjusted corrector under assumptions and a control variate formulation that enables image-scale evaluations. 
    \item We present a scalable quadrature approximation that requires only one additional midpoint score evaluation per proposal based on Simpson's rule.
    \item We demonstrate that MADM correctors reduce artifacts on synthetic datasets and consistently improve FID on standard image benchmarks.
\end{itemize}
\section{Background and Setup}
\subsection{Langevin Diffusion and Unadjusted Langevin Algorithm}\label{subsec:Langevin}
Consider a smooth distribution $p$, then a standard technique for sampling from $p$ is to simulate the Langevin diffusion
\begin{equation}
\rmd \mathbf{X}_t=\tfrac{g^2}{2}\nabla_{\mathbf{x}} \log p(\mathbf{X}_{t})\rmd t 
+ g \rmd \mathbf{B}_{t},
\end{equation}
for some parameters $g>0$. Whatever being the initial state $\mathbf{X}_0$, this diffusion is such that the law of $\mathbf{X}_t$ converges to $p$ as $t$ grows without bound. One cannot simulate exactly this nonlinear continuous-time diffusion process so a standard approach is to use an Euler--Maruyama integrator
\begin{equation}\label{eq:ULA}
\mathbf{X}_{kh}=\mathbf{X}_{(k-1)h}+h\tfrac{g^2}{2} \nabla_{\mathbf{x}}\log p(\mathbf{X}_{(k-1)h})+\sqrt{h} g \mathbf{Z}_k, \qquad \mathbf{Z}_k \sim \mathcal{N}(\mathbf{0}, \mathbf{I}),
\end{equation}
where $h>0$ is a step-size. This is known as the Unadjusted Langevin Algorithm (ULA). This time-discretisation introduces a bias so that we do not sample exactly from $p$.

\subsection{Diffusion Models}\label{sec:diffusion}
A Denoising Diffusion Model (DDM) \citep{sohl-dickstein2015dpm,ho2020ddpm,song2021sgm} uses a \emph{forward SDE} to transport the data distribution $p_{\text{data}}$ to a multivariate normal $\mathcal{N}(\mathbf{0}, \mathbf{I})$ through
\begin{align}\label{equation:sgm_forward_process}
    \rmd \mathbf{X}_{t} = f_t \mathbf{X}_{t} \rmd t + g_t \rmd \mathbf{B}_{t}, && \mathbf{X}_{0} \sim p_{\text{data}},
\end{align}
where $t \in [0,1]$ and $\mathbf{B}_t$ is a multivariate standard Brownian motion. The functions $f_t$ and $g_t$ must ensure that $\mathbf{X}_{1} \sim \mathcal{N}(\mathbf{0}, \mathbf{I})$, at least approximately, but are otherwise design choices which can be selected to optimise sample quality; see, e.g. \citep{williams2024score}. 
The reverse-time process $(\mathbf{Y}_t)_{t\in[0,1]}=(\mathbf{X}_{1-t})_{{t\in[0,1]}}$ thus transports $\mathcal{N}(\mathbf{0}, \mathbf{I})$ to $p_{\text{data}}$ and satisfies 
\begin{align}\label{equation:sgm_time_reversal}
    \rmd \mathbf{Y}_t = [-f_{1-t} \mathbf{Y}_{t} + g_{1-t}^2 \nabla_{\mathbf{x}} \log p_{1-t}(\mathbf{Y}_{t})]\rmd t + g_{1-t} \rmd \mathbf{B}_{t}, \quad \mathbf{Y}_0 \sim \mathcal{N}(\mathbf{0}, \mathbf{I}),
\end{align}
where $p_{t}$ is the distribution of $\mathbf{X}_t$. In practice, the score function $\nabla_{\mathbf{x}} \log p_{t}(\cdot)$ is unknown and is approximated by a neural network $s_{\phi}(\mathbf{x}_t, t)$ trained by denoising score matching \citep{vincent2011dsm,song2019smld}, that is, by regressing $s_{\phi}(\mathbf{X}_t, t)$ onto $\nabla_{\mathbf{x}} \log p_{t|0}(\mathbf{X}_t|\mathbf{X}_0)$ with $\mathbf{X}_0 \sim p_{\text{data}}$. Since \eqref{equation:sgm_forward_process} is linear, this conditional score is available in closed form, as
\begin{align*}
    p_{t|0}(\mathbf{x}_{t}|\mathbf{x}_{0}) = \mathcal{N}(r_t\mathbf{x}_{0}, r^2_t \sigma^2_t \mathbf{I}), && r_t = \exp\left(\int_0^t f_u \rmd u\right), && \sigma^2_t = \int_0^t \frac{g_u^2}{r_u^2} \rmd u.
\end{align*}
Different choices of $f_t$ and $g_t$ in Equation \ref{equation:sgm_forward_process} yield the variance exploding (VE) and variance preserving (VP) SDEs, with $r_t$ and $\sigma_t$ known in closed form (see Table 1 of \cite{karras2022edm}). For convenience, we have assumed the VP setting (i.e. $p_1=\mathcal{N}(\mathbf{0}, \mathbf{I})$). Given $f_t$ and $g_t$, together with a learned score and schedule, the task is to evolve Equation \ref{equation:sgm_time_reversal} to draw samples from the generative model.

\subsection{Diffusion Predictor-Corrector Framework}
Equation \ref{equation:sgm_time_reversal} can be split into the sum of a deterministic drift and a Langevin diffusion, in which the deterministic component defines the predictor step and the stochastic component the corrector:
\begin{equation*}
\label{equation:sgm_predictor_corrector}
\begin{aligned}
\rmd \mathbf{Y}_{t} 
&= \underbrace{\left( -f_{1-t} \mathbf{Y}_{t} 
+ \tfrac{g_{1-t}^2}{2} \nabla_{\mathbf{x}} \log p_{1-t}(\mathbf{Y}_{t}) \right)\rmd t}_{\text{Probability Flow Prediction ODE}} 
+ \underbrace{\left( \tfrac{g_{1-t}^2}{2} \nabla_{\mathbf{x}} \log p_{1-t}(\mathbf{Y}_{t})\rmd t 
+ g_{1-t} \rmd \mathbf{B}_{t} \right)}_{\text{Langevin Correction SDE}}.
\end{aligned}
\end{equation*}

This decomposition gives rise to the predictor--corrector framework where at each step we integrate numerically the ODE before applying the Langevin Correction SDE as described in \cite{song2021sgm}. The rationale is that, if $\mathbf{Y}_{t-\eta} \sim p_{1-(t-\eta)}$ then by integrating exactly the deterministic drift with an exact score over an interval of length $\eta$ then we obtain $\mathbf{Y}'_{t} \sim p_{1-t}$ and the Langevin diffusion initialised at such state leaves this distribution invariant and so provide a new sample $\mathbf{Y}_{t} \sim p_{1-t}$. Despite the terminology, the Langevin ``corrector'' is itself defined by an SDE and must be discretised as discussed in Section \ref{subsec:Langevin}. Using an Euler--Maruyama scheme, this reads 
\begin{align}
\mathbf{Y}'_{k\eta}&=\mathbf{Y}_{(k-1)\eta}+ \eta \Big(-f_{1-(k-1)\eta} \mathbf{Y}_{(k-1)\eta} 
+ \tfrac{g_{1-(k-1)\eta}^2}{2} \nabla_{\mathbf{x}} \log p_{1-(k-1)\eta}(\mathbf{Y}_{(k-1)\eta}) \Big ), \label{eq:predictorODE}\\
\mathbf{Y}_{k \eta}&=\mathbf{Y}'_{k\eta}+\tfrac{g_{1-k\eta}^2}{2} h \nabla_{\mathbf{x}} \log p_{1-k\eta}(\mathbf{Y}'_{k \eta})+\sqrt{h} g_{1-k\eta} \mathbf{Z}, \qquad \mathbf{Z} \sim \mathcal{N}(\mathbf{0}, \mathbf{I}), \label{eq:correctorULA}
\end{align}
where \eqref{eq:correctorULA} is a single ULA step but we can use multiple steps.

In summary, the Langevin corrector is approximated by ULA and so it does not yield exact samples from $p_{1-t}$. While it may reduce errors introduced by the predictor, it introduces its own discretisation bias. In \citep{song2021sgm}, the step size is chosen adaptively, scaling inversely with the squared norm of the score, which in the context of ULA introduces additional bias. In the MCMC literature, explicit mechanisms exist to correct such bias and ensure exactness which we describe in the next section.

\subsection{Markov chain Monte Carlo}\label{sec:mcmc}
We are interested in sampling from a target density \(p\) on $\mathbb{R}^d$, known up to a normalising constant. We describe here a generic MCMC algorithm which can be used to correct the bias of ULA. It generates a Markov chain $(\mathbf{X}^{(i)})_{i \geq 1}$ whose equilibrium distribution is $p$. 

Given the current state $\mathbf{X}^{(i)}=\mathbf{x}$ of the Markov chain, we sample a candidate state $\widetilde{\mathbf{x}}$ from a proposal density $q(\widetilde{\mathbf{x}} \mid \mathbf{x})$. This candidate is accepted with probability $\alpha_{\mathrm{MH}}(\mathbf{x},\widetilde{\mathbf{x}})$ for the MH algorithm and $\alpha_{\mathrm{B}}(\mathbf{x},\widetilde{\mathbf{x}})$ for the Barker algorithm, where
\begin{equation}\label{eq:acceptanceproba}
\alpha_{\mathrm{MH}}(\mathbf{x},\widetilde{\mathbf{x}})
=
1\wedge R(\mathbf{x},\widetilde{\mathbf{x}}),\qquad \alpha_{\mathrm{B}}(\mathbf{x},\widetilde{\mathbf{x}})
=
R(\mathbf{x},\widetilde{\mathbf{x}}) \, (1+R(\mathbf{x},\widetilde{\mathbf{x}}))^{-1},
\end{equation}
with 
\begin{equation}\label{eq:ratiotargetprop}
 R(\mathbf{x},\widetilde{\mathbf{x}})=H(\mathbf{x},\widetilde{\mathbf{x}}) r_p(\mathbf{x},\widetilde{\mathbf{x}}),\qquad
H(\mathbf{x},\widetilde{\mathbf{x}})=\frac{q(\mathbf{x}\mid \widetilde{\mathbf{x}})}{q(\widetilde{\mathbf{x}}\mid \mathbf{x})},\qquad r_p(\mathbf{x},\widetilde{\mathbf{x}})=\frac{p(\widetilde{\mathbf{x}})}{p(\mathbf{x})}.
\end{equation}
If the proposal is accepted then we set $\mathbf{X}^{(i+1)}=\widetilde{\mathbf{x}}$, otherwise $\mathbf{X}^{(i+1)}=\mathbf{X}^{(i)}$.
The popularity of the MH acceptance function stems from the result of \cite{peskun1973optimum}, which shows that, among reversible accept--reject rules, it minimises the asymptotic variance of ergodic averages. However, \cite{latuszynski2013clts} demonstrated that Barker’s acceptance function is not far from optimal: in the worst case, it increases the asymptotic variance of Monte Carlo averages by at most a factor of two compared with MH. Although MCMC algorithms based on Barker’s acceptance function are no longer strictly speaking \emph{Metropolis--Hastings} methods, we retain the term \emph{Metropolis-adjusted} (instead of \emph{Barker-adjusted}) in this paper for simplicity.

Such MCMC techniques can be used to correct for the bias introduced by the ULA algorithm \ref{eq:ULA}. Consider $g=1$ for simplicity, then ULA can be thought of as defining a proposal $q(\widetilde{\mathbf{x}}\mid \mathbf{x})=\mathcal{N}\!\left(\widetilde{\mathbf{x}};\mathbf{x}+\frac{h}{2} \nabla \log p(\mathbf{x}),h \mathbf{I}\right)$. We can thus use this proposal within Metropolis--Hastings to sample (asymptotically in the number of iterations) exactly from $p$. This is a standard approach known as Metropolis-Adjusted Langevin Algorithm (MALA) \citep{besag1994comments,roberts1996exponential}.

In the context of diffusion models, while the proposal-ratio term \(H_t\) is computable from score evaluations, both the MH and Barker's acceptance probability rely on a density ratio
\(p(\widetilde{\mathbf{x}})/p(\mathbf{x})\) which is not accessible as $p$ is the intractable diffused data distribution at diffusion time $t$; i.e., $p=p_t$ in the notation of Section \ref{sec:diffusion}. So a direct implementation of MALA is not available from score evaluations alone. In Section \ref{sec:metropolis-adjusted_diffusion_models}, we show how we can bypass this key issue.

\begin{algorithm}[t]
   \caption{MALA for Denoising Diffusion Models}
   \label{algorithm:mala_ddm}
\begin{algorithmic}
    \State {\bfseries Input:} initial sample $\mathbf{x}_{t}^{(0)}$, score $s(\cdot, t)$, corrector steps $K$, step size $h$.
    \For{$i=0$ {\bfseries to} $K-1$}
    \State Compute $\widetilde{\mathbf{x}}_{t}^{(i+1)}=\mathbf{x}_{t}^{(i)} + \frac{h}{2} s(\mathbf{x}_{t}^{(i)}, t) + \sqrt{h} \mathbf{z}^{(i)}$ with $\mathbf{z}^{(i)} \sim \mathcal{N}(\mathbf{0}, \mathbf{I})$.
    \State Compute ACCEPT/REJECT using Algorithm \ref{algorithm:exact_two_coin} or \ref{algorithm:simpsons_adjustment} with state $\mathbf{x}_{t}$ and proposal $\widetilde{\mathbf{x}}_{t}^{(i+1)}$
    \State Set $\mathbf{x}_{t}^{(i+1)} = \widetilde{\mathbf{x}}_{t}^{(i+1)}$  if ACCEPT; otherwise $\mathbf{x}_{t}^{(i+1)} = \mathbf{x}_{t}^{(i)}$. 
   \EndFor
   \State \Return $\mathbf{x}_{t}^{(K)}$.
\end{algorithmic}
\end{algorithm}

\section{Metropolis-Adjusted Diffusion Models}
\label{sec:metropolis-adjusted_diffusion_models}
In this section, we introduce the Metropolis-Adjusted samplers, intended as correctors for diffusion models, that avoid the problem of non-vanishing bias. The challenge with designing these arises from the fact that the target distribution $p_t$ at time $t$ is unknown and only available through its score $s(t,\mathbf{x}_t)=\nabla_{\mathbf{x}} \log p_t(\mathbf{x})$. In Section \ref{sec:lineintegral}, we re-express this density ratio using a line integral which we then use in Sections~\ref{sec:unbiased_density_ratio_estimator} and \ref{sec:two_coin} to show how one can obtain an unbiased estimate of $r_{p_t}$. We then leverage this to obtain an MCMC algorithm admitting exactly $p_t$ as invariant distribution via a two-coin Bernoulli factory \citep{Goncalves2017-lh}. Section~\ref{sec:quadrature_approximations} provides an alternative lightweight estimate of the integral by a single application of Simpson's rule on the integrand. A comparison of both procedures is provided in Figure \ref{fig:estimators_properties}. 

\subsection{A key density ratio identity}\label{sec:lineintegral}
Consider a ULA proposal to sample from $p_t$; i.e., $q(\widetilde{\mathbf{x}}\mid \mathbf{x})=\mathcal{N}\!\left(\widetilde{\mathbf{x}};\mathbf{x}+\frac{h}{2} s(\mathbf{x},t),h \mathbf{I}\right)$. Then, as shown in Section \ref{sec:mcmc}, Metropolis-Adjusted schemes require evaluations of the ratio of proposals,
\begin{equation*}
   \log  H_t(\mathbf{x}, \widetilde{\mathbf{x}}) =\log \left(\frac{q(\mathbf{x}\mid \widetilde{\mathbf{x}})}{q(\widetilde{\mathbf{x}}\mid \mathbf{x})}\right)= \frac{1}{2h} \big \| \widetilde{\mathbf{x}} - \mathbf{x} - \tfrac{h}{2} \, s(\mathbf{x}, t) \big \|^2 - \frac{1}{2h} \big \| \mathbf{x} - \widetilde{\mathbf{x}} - \tfrac{h}{2} \, s(\widetilde{\mathbf{x}}, t) \big \|^2,
\end{equation*}
as well as the target density ratio $r_{p_t}(\mathbf{x},\widetilde{\mathbf{x}})=
p_t(\widetilde{\mathbf{x}})/ p_t(\mathbf{x})$. While we can easily compute $H_t(\mathbf{x}, \widetilde{\mathbf{x}})$, $r_{p_t}$ can not be obtained as directly. Here, we present two methods to estimate $r_{p_t}$ and ultimately perform the Metropolis-Adjusted Diffusion corrector in Algorithm~\ref{algorithm:mala_ddm}, when only the score $s(\cdot,t)$ is available. Both methods are based on the identity (see Appendix \ref{appendix:density-ratio} for a full derivation):
\begin{equation}
\label{eq:exact_i}
    \log(r_{p_t}(\mathbf{x}, \widetilde{\mathbf{x}})) = \int_0^1 \big \langle s \big ( \mathbf{x} + u(\widetilde{\mathbf{x}} - \mathbf{x}), t \big ), \widetilde{\mathbf{x}} - \mathbf{x} \big \rangle \, \rmd u,
\end{equation}
which reduces the problem of estimating $r_{p_t}(\mathbf{x},\widetilde{\mathbf{x}})$ to estimating a one-dimensional integral.

\subsection{An unbiased estimator of the density ratio}
\label{sec:unbiased_density_ratio_estimator}
We now construct a bounded unbiased estimator of the intractable density ratio
\(r_{p_t}(\mathbf{x},\widetilde{\mathbf{x}})\). This is the key object needed by the Bernoulli factory presented in Section~\ref{sec:two_coin} which will be used to perform exact adjustment. From \eqref{eq:exact_i}, a natural unbiased estimator of the line integral is
\begin{equation}
    I = \left\langle s(\mathbf{x}+U(\widetilde{\mathbf{x}}-\mathbf{x}),t), \widetilde{\mathbf{x}}-\mathbf{x} \right\rangle, \quad U \sim \mathcal{U}([0,1]),
\end{equation}
since we have \(\mathbb E[I]= \log r_{p_t}(\mathbf{x},\widetilde{\mathbf{x}})\). However, directly exponentiating
\(I\) does not give an unbiased estimator of
\(r_{p_t}(\mathbf{x},\widetilde{\mathbf{x}})\), and furthermore it does not
produce a quantity bounded in \([0,1]\) as required of the Bernoulli factory technique. Let $\hat{I}_{\Delta}^{(2)}=\tfrac{1}{2} \langle s(\widetilde{\mathbf{x}},t) + s(\mathbf{x},t), \widetilde{\mathbf{x}} - \mathbf{x} \rangle$ be the trapezoidal estimate of \eqref{eq:exact_i}.
The following analysis can be easily adjusted for other deterministic integral estimates $\hat{I}_{\Delta}$, but for clarity we fix $\hat{I}_{\Delta} = \hat{I}_{\Delta}^{(2)}$ throughout. 
To ensure boundedness, we make the following assumption. 
\begin{assumption} \label{ass:score_bound}
There exists a function $C:\mathbb R^d\times\mathbb R^d\to\mathbb R_+$ such that, for every
$\mathbf{x},\widetilde{\mathbf{x}}\in\mathbb R^d$,
\begin{equation}
    \max_{u\in[0,1]} \left| \left\langle s(\mathbf{x}+u(\widetilde{\mathbf{x}}-\mathbf{x}),t), \widetilde{\mathbf{x}} -\mathbf{x} \right\rangle - \hat{I}_{\Delta}^{(2)} \right| \le C(\mathbf{x},\widetilde{\mathbf{x}}).
\end{equation}
\end{assumption}

In the following lemma, we show that this assumption holds for any data distribution, with a crude choice of $C$ that only requires the score evaluations $s(\mathbf{x}, t)$ and $s(\widetilde{\mathbf{x}}, t)$ (see proof in Appendix~\ref{app:proof_bernoulli}).
\begin{lemma}
\label{lem:control_variate_exact_i}
Suppose that \(s(\cdot,t)=\nabla\log p_t(\cdot)\). Then Assumption~\ref{ass:score_bound} holds with
\begin{equation*}
    C(\mathbf{x},\widetilde{\mathbf{x}}) = \tfrac{1}{2} \big \langle s(\widetilde{\mathbf{x}}, t) - s(\mathbf{x}, t), \widetilde{\mathbf{x}} - \mathbf{x} \big \rangle + \tfrac{1}{r_t^2 \sigma_t^2}\|\widetilde{\mathbf{x}} - \mathbf{x}\|^2.
\end{equation*}
\end{lemma}
In the following section, we will show that the efficiency of the proposed algorithm depends on the tightness of the estimate $C$. In Appendix~\ref{appendix:implementation_considerations_c}, we show that this estimate can be made less crude using smoothness in the score function.

Under Assumption~\ref{ass:score_bound}, we obtain $|I - \hat{I}_{\Delta}^{(2)}| \leq C(\mathbf{x},\widetilde{\mathbf{x}})$ almost surely and the truncation error of $\hat{I}_{\Delta}^{(2)}$ 
\begin{align*}
    \E[I - \hat{I}_{\Delta}^{(2)}] = \int_0^1 \left[\langle s \big ( \mathbf{x} + u(\widetilde{\mathbf{x}} - \mathbf{x}), t \big ), \widetilde{\mathbf{x}} - \mathbf{x} \big \rangle - \hat{I}_{\Delta}^{(2)} \right] \rmd u .
\end{align*}
With $\hat{I}_{\Delta}^{(2)}$ as a control variate, we obtain unbiased estimates of Equation \ref{eq:exact_i} with unbiased estimates of truncation error. To estimate the exponential of this integral, we use a Poisson-truncated power-series estimator. Let
\(N\sim\mathrm{Poisson}(2C(\mathbf{x},\widetilde{\mathbf{x}}))\) and
\(I_1,\ldots,I_N\) be independent copies of \(I\), and define
\begin{equation}
    W =\prod_{j=1}^{N} \left( \frac12+\frac{I_j - \hat{I}_{\Delta}^{(2)}}{2C(\mathbf{x},\widetilde{\mathbf{x}})} \right).
\end{equation}
The intuition is that the random product expands the exponential series for
\(\exp\{\log r_{p_t}(\mathbf{x},\widetilde{\mathbf{x}})-\hat{I}_{\Delta}^{(2)}-C(\mathbf{x},\widetilde{\mathbf{x}})\}\), while the rescaling by
\(C(\mathbf{x},\widetilde{\mathbf{x}})\) keeps each factor within $[0, 1]$. In the following lemma, we show that this indeed produces an unbiased estimator of $r_{p_t}$ (see proof in Appendix \ref{appendix:proof_poisson_density_ratio_estimator}).

\begin{lemma}
\label{lem:poisson_density_ratio_estimator}
Suppose that Assumption~\ref{ass:score_bound} holds, then \(W\in[0,1]\) and $e^{C(\mathbf{x},\widetilde{\mathbf{x}}) + \hat{I}_{\Delta}^{(2)}}\mathbb E[W] = r_{p_t}(\mathbf{x},\widetilde{\mathbf{x}})$.
\end{lemma}

Hence, conditional on
\(W\), if we sample \(V\sim \mathcal{U}([0,1])\) and declare an accept when \(V\le W\), then the
unconditional accept probability is \(\E[W]\). Thus the construction gives access to a coin with unknown win probability $e^{-[C(\mathbf{x},\widetilde{\mathbf{x}}) + \hat{I}_{\Delta}^{(2)}]}r_{p_t}(\mathbf{x},\widetilde{\mathbf{x}})$,
without evaluating the density ratio itself. The next subsection uses this coin to obtain an exact
Barker accept--reject algorithm.

\begin{figure}[t]
\begin{minipage}[t]{0.5\textwidth}
\begin{algorithm}[H]
\caption{Exact MCMC via two coin algorithm}
\label{algorithm:exact_two_coin}
\begin{algorithmic}
    \State {\bfseries Input:} state $\mathbf{x}$, proposal $\widetilde{\mathbf{x}}$, score function $s(\cdot, t)$, step size $h$, bound $C(\cdot, \cdot)$.
    \State Set $\alpha' = (1 + H_t(\mathbf{x}, \widetilde{\mathbf{x}}) \exp(\hat{I}_{\Delta}^{(2)} + C(\mathbf{x}, \widetilde{\mathbf{x}})))^{-1}$
    \While{True}
        \If{$U \leq \alpha'~\text{for}~U \sim \mathcal{U}([0,1])$}
            \State \Return REJECT.
        \EndIf
        \State Sample $N \sim \operatorname{Poisson}(2 C(\mathbf{x}, \widetilde{\mathbf{x}}))$
        \For{$i=1$ {\bfseries to} $N$, $W \gets 1$}
            \State Sample $U \sim \mathcal{U}([0,1])$.
            \State $I \gets \langle s(\mathbf{x} + U (\widetilde{\mathbf{x}} - \mathbf{x}), t), \widetilde{\mathbf{x}} - \mathbf{x} \rangle$.
            \State $W \gets W \cdot \Big ( \frac{1}{2} + \frac{I - \hat{I}_{\Delta}^{(2)}}{2 C(\mathbf{x}, \widetilde{\mathbf{x}})} \Big )$.
        \EndFor
        \If{$U \leq W~\text{for}~U \sim \mathcal{U}([0,1])$}
            \State \Return ACCEPT.
        \EndIf
    \EndWhile
\end{algorithmic}
\end{algorithm}
\end{minipage}
\hfill
\begin{minipage}[t]{0.49\textwidth}
\begin{algorithm}[H]
   \caption{MH estimate by Simpson's 1/3 rule}
   \label{algorithm:simpsons_adjustment}
\begin{algorithmic}
    \State {\bfseries Input:} state $\mathbf{x}$, proposal $\widetilde{\mathbf{x}}$, score function $s(\cdot, t)$, step size $h$, midpoint $\mathbf{x}_{\texttt{mid}} = \frac{1}{2}(\mathbf{x} + \widetilde{\mathbf{x}})$.
    \State Compute $\widetilde{s} = \frac{1}{6}(s(\mathbf{x}, t) + 4 s(\mathbf{x}_{\texttt{mid}}, t) + s(\widetilde{\mathbf{x}}, t))$.
    \State Compute $\widehat{I} = \langle \widetilde{s}, \widetilde{\mathbf{x}} - \mathbf{x}\rangle$.
    \State Compute $\widehat{\alpha} = \min\{1, \exp(\widehat{I}) H_t(\mathbf{x}, \widetilde{\mathbf{x}})\}$.
    \State Sample $U \sim \mathcal{U}([0,1])$.
    \If{$U \leq \widehat{\alpha}$}
        \State \Return ACCEPT.
    \EndIf
    \State \Return REJECT.
\end{algorithmic}
\end{algorithm}

\centering
\begin{tabular}{lcc}
    \toprule
    Estimator & Unbiased & Low Variance \\
    \midrule
    Alg. \ref{algorithm:exact_two_coin} & $\checkmark$  & $\times$     \\
    Alg. \ref{algorithm:simpsons_adjustment} & $\times$ & $\checkmark$      \\
    \bottomrule
\end{tabular}
\caption{Properties of integral estimators.}
\label{fig:estimators_properties}
\end{minipage}
\end{figure}

\subsection{A two-coin MCMC algorithm}
\label{sec:two_coin}

We now turn the bounded estimator from Section~\ref{sec:unbiased_density_ratio_estimator} into an
exact Barker correction. A Bernoulli factory is a randomized procedure which produces a Bernoulli
random variable with a desired probability \(f(w)\), while only having access to flips of a coin with
unknown probability \(w\) \citep{keane1994bernoulli,nacu2005fast}. Here, the \(w\)-coin is generated by the random variable \(W\) in
Lemma~\ref{lem:poisson_density_ratio_estimator}, and the desired probability is the Barker
acceptance probability $\alpha_{\mathrm B}(\mathbf{x},\widetilde{\mathbf{x}})$ (see \eqref{eq:acceptanceproba}).
To simplify notation, we write $H=H_t(\mathbf{x},\widetilde{\mathbf{x}})$, $C=C(\mathbf{x},\widetilde{\mathbf{x}})$ and $r=r_{p_t}(\mathbf{x},\widetilde{\mathbf{x}})$. The two-coin construction proceeds in repeated rounds. In each round, we first reject immediately
with probability,
\begin{equation}\label{eq:alpha_prime}
\mathbb P(\text{reject in round}) = \alpha'
=
(1+He^{\hat{I}_{\Delta}^{(2)} + C})^{-1}.
\end{equation}
If immediate rejection does not occur, we accept with probability \(W\). This gives us
\begin{align*}
    \mathbb P(\text{accept in round}) = (1-\alpha')e^{-[\hat{I}_{\Delta}^{(2)} + C]}r_{p_t}, \quad
\mathbb P(\text{restart}) = (1-\alpha')(1-e^{-[\hat{I}_{\Delta}^{(2)} + C]}r_{p_t}).
\end{align*}

Summing over the resulting geometric sequence of restarts gives the eventual acceptance probability $\alpha_{\mathrm B}(\mathbf{x},\widetilde{\mathbf{x}})$.
Hence the algorithm simulates the exact Barker accept--reject decision without evaluating
\(r_{p_t}(\mathbf{x},\widetilde{\mathbf{x}})\). We state this result rigorously in the following theorem (see proof in Appendix \ref{proof:two_coin_barker}).
\begin{theorem}
\label{thm:two_coin_barker}
Suppose that \(s(\cdot,t)=\nabla\log p_t(\cdot)\) for some positive differentiable density $p_t$ and that Assumption~\ref{ass:score_bound} holds. Then Algorithm~2 almost surely terminates in finitely many iterations and returns an accept with probability $\alpha_{\mathrm B}(\mathbf{x},\widetilde{\mathbf{x}})$ (see \eqref{eq:acceptanceproba}).
%
Consequently, the adjusted proposal in Algorithm \ref{algorithm:mala_ddm} using Algorithm \ref{algorithm:exact_two_coin} is reversible w.r.t. \(p_t\) and hence, the algorithm produces samples that converge in total variation distance to \(p_t\).
\end{theorem}

To determine the number of iterations required for Algorithm~\ref{algorithm:exact_two_coin} to terminate, we state the following proposition which characterises its iteration complexity (see proof in Appendix \ref{proof:two_coin_cost}). 

\begin{proposition}
\label{prop:two_coin_cost}
In the setting of Theorem~\ref{thm:two_coin_barker}, the number of iterations before Algorithm~\ref{algorithm:exact_two_coin} terminates is distributed as a geometric distribution with probability $(1+Hr)/(1+He^{\hat{I}_{\Delta}^{(2)} + C})$.
Furthermore, the number of score queries has expected value $(2CHe^{\hat{I}_{\Delta}^{(2)} + C})/(1+Hr)$ and, with probability at least $1 - \delta$, it is at most,
\begin{equation*}
    4C \frac{1+He^{\hat{I}_{\Delta}^{(2)} + C}}{1+Hr} \log(2/\delta) + 4C + \log(2/\delta).
\end{equation*}
\end{proposition}

The computational cost of Algorithm~\ref{algorithm:exact_two_coin} is governed by the tightness of $C$. Since $|\log r - \hat{I}_{\Delta}^{(2)}| \leq C$, we have $He^{\hat{I}_{\Delta}^{(2)} + C} \leq e^{2C}(1 + Hr)$, so the expected number of score queries is at most $2Ce^{2C}$. A loose bound is therefore costly, while for small Langevin step sizes $h$ the regime $C \ll 1$ arises naturally. For the choice of $C$ in Lemma~\ref{lem:control_variate_exact_i}, $C$ scales with $\|\widetilde{\mathbf{x}} - \mathbf{x}\|^2 = O(hd)$, so the expected number of score queries is $O(hd)$, decaying to zero as the step size is reduced.

\subsection{Newton--Cotes quadrature approximation}
\label{sec:quadrature_approximations}
While Algorithm \ref{algorithm:exact_two_coin} gives an exact Barker correction, it requires a random number of score evaluations. A cheaper alternative is to approximate the line integral in \eqref{eq:exact_i} deterministically with a Newton--Cotes quadrature rule, and then plug the result into the usual Metropolis--Hastings acceptance probability. 
A Newton--Cotes quadrature rule with $N+1$ equally spaced points approximates this integral using
\begin{equation}
\label{eq:quadrature_ihat}
    \hat{I}_{\Delta}^{(N+1)} = \sum_{i=0}^{N} w_i \left\langle s\left(\mathbf{x} + \tfrac{i}{N}(\widetilde{\mathbf{x}} - \mathbf{x}), t\right), \widetilde{\mathbf{x}} - \mathbf{x} \right \rangle,
\end{equation}
where $i=0,1,\dots,N$ and $\sum_{i=0}^{N} w_i =1$. This gives rise to well-known quadrature rules such as the trapezoidal rule and Simpson's 1/3 rule, where truncation errors are described by derivatives of the function $f(u) = \big \langle s \big ( \mathbf{x} + u(\widetilde{\mathbf{x}} - \mathbf{x}), t \big ), \widetilde{\mathbf{x}} - \mathbf{x} \big \rangle$ (i.e. the integrand of Equation \eqref{eq:exact_i}). For a survey of Newton-Cotes quadrature rules and their truncation errors, we refer the reader to \cite{atkinson1989numerical}.
\begin{example}[Trapezoidal Rule]
\label{example:trapezoidal_rule}
\normalfont
Setting $N=1$ and $w_0=w_1=\frac{1}{2}$ yields the estimate
\begin{align*}
    \hat{I}_{\Delta}^{(2)} = \frac{1}{2} \langle s(\mathbf{x}, t) + s(\tilde{\mathbf{x}},t), \tilde{\mathbf{x}} - \mathbf{x} \rangle.
\end{align*}
This has truncation error $-\frac{1}{12} f^{\prime \prime}(u^{*})$ for some $u^{*} \in [0,1]$, and is of order $\mathcal{O}(h^{3/2})$.
\end{example}
\begin{example}[Simpson's 1/3 Rule]
\label{example:simpsons_13_rule}
\normalfont
Setting $N=2$, $w_0=w_2=\frac{1}{6}$ and $w_1=\frac{4}{6}$ yields the estimate
\begin{align*}
    \hat{I}_{\Delta}^{(3)} = \frac{1}{6} \langle s(\mathbf{x}, t) + 4s(\mathbf{x}_{\texttt{mid}}, t) + s(\widetilde{\mathbf{x}},t), \widetilde{\mathbf{x}} - \mathbf{x} \rangle, \qquad \mathbf{x}_{\texttt{mid}} = \frac{1}{2}(\mathbf{x} + \widetilde{\mathbf{x}}).
\end{align*}
This has truncation error $-\frac{1}{2880} f^{(4)}(u^{*})$ for some $u^{*} \in [0,1]$, and is of order $\mathcal{O}(h^{5/2})$.
\end{example}
With an estimate from Equation \ref{eq:quadrature_ihat} in hand, we can use it in place of $r_{p_t}$ in the Metropolis-Hastings acceptance probability $\hat{\alpha}_{\operatorname{MH}}(\mathbf{x}, \widetilde{\mathbf{x}}) = \min \big \{ 1, \exp ( \hat{I}_{\Delta} ) H_t(\mathbf{x}, \widetilde{\mathbf{x}}) \big \}$. Empirically, we found Simpson's 1/3 rule effective. This yields Algorithm \ref{algorithm:simpsons_adjustment}. Other Newton--Cotes quadrature rules could be used to approximate $\hat{I}$.
The estimate only requires access to the score and enjoys two advantages over the exact method. The first is that it is more efficient to compute as other quantities including the weights $\{w_i\}_{i=0}^{N}$ are known. The second is that it provides a closed form estimate of the acceptance probability. The main disadvantage is that it is a biased estimate of Equation \eqref{eq:exact_i} due to truncation.


\subsection{Implementation considerations}


\textbf{Practical termination.} To bound the random evaluation cost of Algorithms  \ref{algorithm:exact_two_coin}, we run it for at most $K$ rounds. If neither an accept nor reject decision is reached, we default to Algorithm \ref{algorithm:simpsons_adjustment}. This forms a hybrid sampling algorithm (see Algorithm \ref{algorithm:combined_accept_reject} in Appendix \ref{sec:madm_hybrid}).

\textbf{Choosing $h$.} The step size controls a familiar tradeoff between mixing speed and rejection rate. We show that in the high-dimensional limit with a Gaussian target, the optimal step size scales as $h \sim d^{-1/3}$ and yields an asymptotic acceptance rate of $0.347$ (see Appendix~\ref{app:proof_optimal_scaling}). This rate provides a practical heuristic: tune $h$ to achieve an observed acceptance rate of roughly one-third.

\section{Related work}

\textbf{Bernoulli Factories.}
For diffusion models, \citep{chen2026sampling} used a Bernoulli factory to sample from $p_{t|t+\delta}(x|x')\propto p_t(x)\exp(-||x-\alpha_t  x'||^2/(2\sigma^2_t))$ when $p_t$ is only known through its score. This \emph{couples} each predictor step with a rejection sampling step, with the latter acting as the corrector. This differs from our work as MADM corrector steps are optional and separate from predictor steps. Bernoulli factories have also been used in MCMC to sample intractable posteriors \citep{herbei2014estimating,Goncalves2017-lh,morina2022bernoulli,vats2022efficient,gonccalves2023exact}.

\textbf{Soft Metropolis--Hastings Adjustments.}
Under the assumption that $p_t \approx p_{t+1}$, \cite{feng2025softmh} similarly suggested using each predictor step as a proposal, and \emph{interpolating} between the proposal and current state with the MH acceptance probability. The latter relaxes the hard update from MH adjustments to a \emph{soft} update. This differs from our work as we make no assumptions on consecutive marginals and apply MH adjustments to \emph{corrector} steps (i.e., $\mathbf{x}_{t}^{(i+1)}$ only targets $p_t$).

\textbf{Metropolis--Hastings Adjustments for Compositional Models.}
\cite{sjoberg2026composition} estimated log density ratios for MH adjustments in compositional models (e.g. guided diffusion) using the trapezoidal rule over Hamiltonian Monte Carlo (HMC) curves and line segments (Equation \ref{eq:exact_i}). The former uses HMC proposals, instead of Langevin proposals as in our work. While both our work and theirs estimate Equation \ref{eq:exact_i} with Newton-Cotes quadratures, we note their FIDs on ImageNet are not competitive and comparisons against other image baselines in \cite{karras2022edm} are not provided. Separately, \cite{du2023ebm} uses an energy parameterisation to enable MH adjustments. These differ from our work as we extend the Predictor-Corrector framework of \cite{song2021sgm} with score-based MH adjustments to corrector steps for \emph{unconditional} and class-conditional models, and show their effectiveness through consistent gains in FID over deterministic sampling baselines.

\textbf{Learning Acceptance Probabilities.}
\cite{aloui2025scorebalance} proposed \emph{learning} acceptance probabilities $\alpha_{\textup{MH}}$ by training an \emph{additional} neural network $\alpha_{\theta}$ with a loss to enforce detailed balance. This differs from MADM as calculations of acceptance probabilities are either bypassed (Algorithm \ref{algorithm:exact_two_coin}) or performed directly with suitable approximations of the density ratio (Algorithm \ref{algorithm:simpsons_adjustment}); enabling the use of MH adjustments while only requiring access to a pre-trained score model.

\section{Convergence guarantees under log-concavity}
\label{sec:theory}

We now analyse the efficiency of the MADM corrector when used within a predictor--corrector sampler, in the simple setting where the data distribution is smooth and strongly log-concave and the exact score $s(\cdot, t) = \nabla \log p_t$ is available. We analyse the sampler that uses $N$ noise levels $1 \geq t_0 > \dots > t_N = 0$ where at each level $t_k$, it applies one step of the ancestral sampler of \citet{ho2020ddpm} as the predictor, followed by $n_k$ steps of the MADM corrector, that is, Algorithm~\ref{algorithm:mala_ddm} with step size $h$ and accept--reject steps given by Algorithm~\ref{algorithm:exact_two_coin}. We write $\nu_N$ for the law of its output.

In the following result, $\widetilde{\Theta}$ and $\widetilde{\mathcal{O}}$ hide factors that are polylogarithmic in $d$, $\kappa$, $\E \|\mathbf{X}_0\|^2$.

\begin{proposition}
\label{prop:main_complexity}
Suppose that $\log p_{\text{data}}$ is twice continuously differentiable and satisfies $m \mathbf{I} \preceq -\nabla^2 \log p_{\text{data}} \preceq L \mathbf{I}$ for constants, $0 < m \leq 1 \leq L$, let $\kappa = L/m$ and let $\varepsilon, \delta \in (0, 1/2)$. Then, with $N = \widetilde{\Theta}(L \sqrt{d} \log(\varepsilon^{-1} \delta^{-1}))$ there exists noise levels such that taking,
\begin{equation*}
    h^{-1} = \widetilde{\Theta} ( L (d + \sqrt{\kappa}) \log^2(\varepsilon^{-1} \delta^{-1}) ), \qquad n_k = \widetilde{\Theta} \big ( \kappa (d + \sqrt{\kappa}) \log^3(\varepsilon^{-1} \delta^{-1}) \big ),
\end{equation*}
we have $\|\nu_N - p_{\text{data}}\|_{\mathrm{TV}} \leq \varepsilon$ and, with probability at least $1 - \delta$, the number of score queries is at most $\widetilde{\mathcal{O}}(L \kappa \sqrt{d} (d + \sqrt{\kappa}) \log^5(\varepsilon^{-1} \delta^{-1}))$.
\end{proposition}

This follows from Propositions~\ref{prop:end_to_end} and \ref{prop:oracle_complexity} applied with accuracy $\varepsilon \wedge \delta/(2N)$; see Appendix~\ref{appendix:oracle_complexity} for the choice of noise levels and explicit constants. The key feature of Proposition~\ref{prop:main_complexity} is its polylogarithmic dependence on $1/\varepsilon$. In comparison, diffusion samplers built from predictor-only steps require $\widetilde{\mathcal{O}}(d/\varepsilon^2)$ score queries \citep{Benton2024-np} or $\widetilde{\mathcal{O}}(L^2 \sqrt{d}/\varepsilon)$ using unadjusted corrections \citep{chen2023probability}. The polynomial dependence in these bounds is a result of the discretisation bias in Euler-type SDE discretisations. In contrast, each MADM corrector step is exactly reversible with respect to $p_{t_k}$, so $\varepsilon$ dependence enters only through the mixing time of the corrector. The same separation can be seen between ULA and MALA in the sampling literature \citep{durmus2017nonasymptotic,Wu2022-fs}. The worse dependence of our bound on $d$ is also fundamental and is a result of the additional complexity of the Bernoulli factory estimator and the warm start requirement for MALA. \citet{chen2026sampling} also recently obtained polylogarithmic guarantees by instead removing the bias of each predictor step with rejection sampling, obtaining a query complexity of $\widetilde{\mathcal{O}}(\sqrt{d} \log^3(\varepsilon^{-1}))$.

\section{Experiments}

\subsection{Protocol}
Our proposed MADM corrector samplers only require access to a score function. This makes them readily applicable to any pre-trained diffusion model. The aim of our experiments was to illustrate the inherent gap in sample quality in existing PC samplers, and to demonstrate the effectiveness of MADM corrections to remedy this. We considered two types of datasets for our experiments: 2D datasets and image datasets. On 2D datasets, we showed that the ancestral sampler generated samples that are off the support of underlying data distributions, and that MADM corrector steps move these outliers back to the support. On image datasets, we demonstrated numerical improvements in FID by interweaving probability flow ODE steps and MADM corrector steps. 

\begin{figure}[t]
    \centering
    \includegraphics[width=0.9\textwidth]{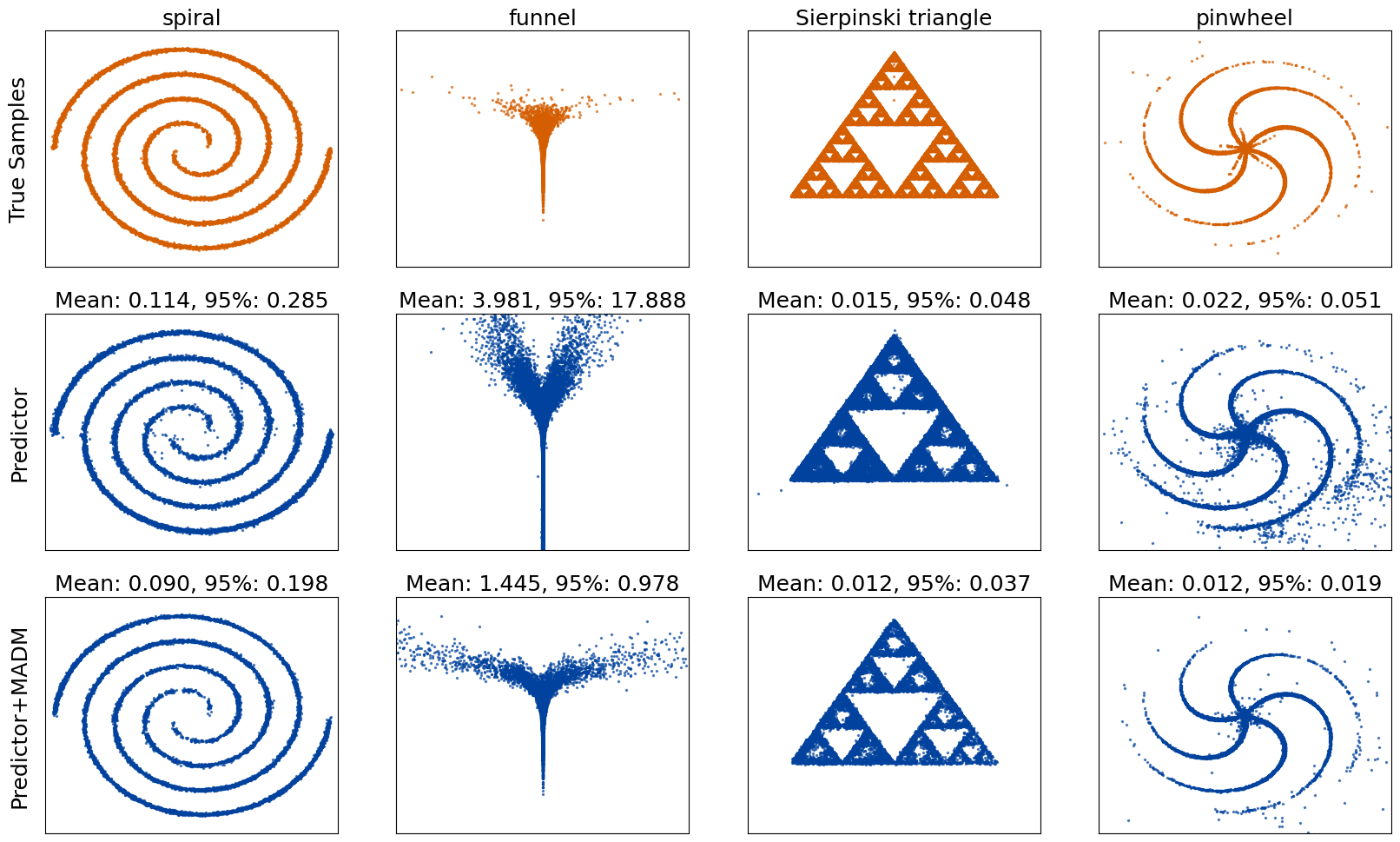}
    \caption{DDM sampling using an ancestral sampler predictor, and MADM correction. Mean distance to the true samples and the 95\% sample containment distance are shown.}
    \label{figure:2d_plots}
\end{figure}

\subsection{Synthetic 2D Datasets}
We considered the ancestral sampler as the predictor on four 2D datasets: spiral, funnel, Sierpinski triangle, and pinwheel. It is known that the ancestral sampler is a specific discretisation of Equation \ref{equation:sgm_time_reversal} (see Appendix E of \cite{song2021sgm}). In other words, we can view one step of the ancestral sampler as equivalent to: one probability flow ODE predictor step and one \emph{unadjusted} Langevin corrector step (i.e. ULA) with a shared step size. Our results with these settings are shown in Figure \ref{figure:2d_plots}. The second row in Figure \ref{figure:2d_plots} illustrates that the ULA step embedded within each ancestral sampler step is not sufficient to move generated samples to the support of underlying data distributions, leaving a large number of outliers. The third row in Figure \ref{figure:2d_plots} shows that MADM corrector steps allow us to move samples to the support of underlying data distributions, leaving few to no outliers. This illustrates an improvement in sample quality by adding accept-reject adjustments within the sampling process.

\subsection{Image Datasets}
We considered the probability flow ODE as the predictor on four standard image datasets: CIFAR-10 \citep{krizhevsky2009learning}, FFHQ \citep{karras2019gans}, AFHQv2 \citep{choi2026stargan} and ImageNet-64 \citep{deng2009imagenet}. We used the pre-trained models and default sampling strategy from \cite{karras2022edm} and evaluated sample quality using FID. We compared against Heun and Euler ODE solvers. The former served as the natural baseline that balances efficiency and sample quality, while the latter served as a diagnostic to verify that corrector steps improve sample quality when an appreciable error gap exists. Finally, we followed the experimental setup for predictor-corrector (PC) samplers in \cite{song2021sgm}, by using \emph{one} corrector step per predictor step. We present numerical comparisons for ULA, MADM using Algorithms \ref{algorithm:simpsons_adjustment} and \ref{algorithm:combined_accept_reject}, and the probability flow ODE with predictor steps matched to the number of function evaluations (NFE) as the PC sampler equipped with Algorithm \ref{algorithm:simpsons_adjustment} as its MADM corrector. Our FID results with these settings are shown in Table \ref{table:image_dataset_results}.

\textbf{Unconditional models on CIFAR-10, FFHQ and AFHQv2.} In the Euler regime, additional predictor or MADM corrector steps both improve FID, though predictor steps are more beneficial as predictor error dominates in this setting. In the Heun regime, additional predictor steps can deteriorate FID, but one MADM corrector step yielded FIDs competitive to those in \cite{karras2022edm,hatamizadeh2025diffit} where gains were consistent, albeit sometimes modest. This suggests that MADM improves or matches the sample quality from deterministic sampling.

\textbf{Class-conditional model on ImageNet-64.} Applying one MADM corrector step per Euler predictor step yielded larger gains in FID than using only Heun predictor steps without corrector steps. This suggests that it is more beneficial to use accept-reject adjustments, rather than a higher-order ODE solver, to mitigate discretisation errors when the number of discretisation steps is large.

\begin{table}[t]
\caption{Fr\'echet Inception Distance (FID) on image datasets sampled using pre-trained models from \cite{karras2022edm}. We consider the probability flow ODE (PF-ODE) predictor equipped with unadjusted Langevin algorithm (ULA) and Metropolis-Adjusted Diffusion Model (MADM) correctors. All FIDs are calculated using 50,000 samples. We highlight the best FID in bold.}
\label{table:image_dataset_results}
\begin{center}
\begin{tabular}{lccccc}
\toprule
ODE Solver & PC Sampler & CIFAR-10 & FFHQ & AFHQv2 & ImageNet\\ 
\midrule
Euler & PF-ODE & 7.62 & 4.59 & 2.93 & 2.41\\
& PF-ODE + ULA & 7.00 & 4.99 & 2.85 & 3.83 \\
& PF-ODE + MADM (Alg. \ref{algorithm:combined_accept_reject}) & 7.53 & 4.58 & 2.93 & \textbf{2.21} \\
& PF-ODE + MADM (Alg. \ref{algorithm:simpsons_adjustment}) & 7.52 & 4.52 & 2.92 & 2.25\\
& PF-ODE (NFE-matched) & \textbf{2.69} & \textbf{2.78} & \textbf{2.19} & 2.34 \\
\midrule
Heun & PF-ODE & 1.96 & 2.47 & 2.04 & 2.32\\
& PF-ODE + ULA & 1.97 & 2.94 & 2.11 & 3.64 \\
& PF-ODE + MADM (Alg. \ref{algorithm:combined_accept_reject}) & 1.96 & 2.45 & 2.04 & 2.16\\
& PF-ODE + MADM (Alg. \ref{algorithm:simpsons_adjustment}) & \textbf{1.95} & \textbf{2.43} & \textbf{2.03} & \textbf{2.15}\\
& PF-ODE (NFE-matched) & 2.01 & 2.45 & 2.05 & 2.32 \\

\bottomrule
\end{tabular}
\end{center}
\end{table}

\section{Conclusion}
\label{sec:conclusion}
We have introduced MADM as a Metropolis-adjusted corrector framework that naturally extends the Predictor-Corrector framework of diffusion models, to mitigate discretisation bias. We presented the first exact method to compute Barker-adjusted Langevin corrections for diffusion models, based on a two-coin Bernoulli factory algorithm. We also presented an efficient approximation to compute Metropolis-adjusted Langevin corrections, based on Simpson's 1/3 rule. We then drew the connection between the two approaches with the trapezoidal approximation serving as a control variate for the exact method. Empirically, we demonstrated effectiveness of MADM for improving sample quality on both synthetic and image datasets. In the latter, MADM corrector steps were more beneficial when the number of predictor steps is large, yielding competitive FIDs with gains that were consistent, albeit sometimes modest. Regarding limitations, our work assumes perfect score estimation and that the exact method can be computationally expensive due to the random number of score evaluations required. We note that the former is not unique to our work and is an inherent limitation of diffusion models, and the latter can be substantially mitigated with the control variate formulation and hybrid sampling approach. Future work can focus on developing efficient subroutines for our exact method.

\section*{Acknowledgements}
\label{sec:acknowledgements}
We acknowledge the use of resources provided by the Isambard-AI National AI Research Resource (AIRR) \citep{mcintosh-smith2024isambard}. Isambard-AI is operated by the University of Bristol and is funded by the UK Government’s Department for Science, Innovation and Technology (DSIT) via UK Research and Innovation; and the Science and Technology Facilities Council [ST/AIRR/I-A-I/1023]. Kevin H. Lam gratefully acknowledges his PhD funding from Google DeepMind. Tyler Farghly is supported by the European Research Council Starting Grant DYNASTY - 101039676. Christopher Williams acknowledges support from: the Defence Science and Technology (DST) Group Australia; and supported in part, by the Engineering and Physical Sciences Research
Council (EPSRC) through the AI Hub in Generative Models [grant number EP/Y028805/1]. Jun Yang acknowledges support from the Independent Research Fund Denmark (DFF) through the Sapere Aude Starting Grant (5251-00032B).


\bibliography{main.bib}
\bibliographystyle{plainnat} 

\newpage


\appendix
\section{Derivation of density ratio identity}
\label{appendix:density-ratio}
Recall that $r_{p_t}(\mathbf{x},\widetilde{\mathbf{x}})=\frac{p_t(\widetilde{\mathbf{x}})}{p_t(\mathbf{x})}$, so it follows that $\log r_{p_t}(\mathbf{x},\widetilde{\mathbf{x}})=\log p_t(\widetilde{\mathbf{x}}) - \log p_t(\mathbf{x})$. Note that $p$ is differentiable. By the fundamental theorem of calculus, we can write
\begin{align*}
    \log p_t(\widetilde{\mathbf{x}}) - \log p_t(\mathbf{x}) &= \int_0^1 \big \langle \nabla \log p_t \big ( \mathbf{x} + u(\widetilde{\mathbf{x}} - \mathbf{x}) \big ), \widetilde{\mathbf{x}} - \mathbf{x} \big \rangle \, \rmd u \\
    &= \int_0^1 \big \langle s \big ( \mathbf{x} + u(\widetilde{\mathbf{x}} - \mathbf{x}), t \big ), \widetilde{\mathbf{x}} - \mathbf{x} \big \rangle \, \rmd u,
\end{align*}
where the second equality follows from setting $s(\cdot, t)=\nabla \log p_t$. This completes our derivation of Equation \eqref{eq:exact_i}.

\section{Proofs for the exact estimator}\label{app:proof_bernoulli}
For the sake of brevity, we use the shorthand $H = H_t(\mathbf{x}, \tilde{\mathbf{x}})$, $C = C(\mathbf{x}, \tilde{\mathbf{x}})$, $r = r_{p_t}(\mathbf{x}, \tilde{\mathbf{x}})$.

\subsection{Proof of Lemma \ref{lem:control_variate_exact_i}}
As in Section~\ref{sec:quadrature_approximations}, let $f(u) = \langle s(\mathbf{x} + u(\widetilde{\mathbf{x}} - \mathbf{x}), t), \widetilde{\mathbf{x}} - \mathbf{x} \rangle$ denote the integrand of \eqref{eq:exact_i}, so that $\hat{I}_{\Delta}^{(2)} = \frac{1}{2}(f(0) + f(1))$.

Recall that $p_t$ is the law of $X_t = r_t X_0 + r_t \sigma_t \pmb{\epsilon}$, where $X_0 \sim p_{\text{data}}$ and $\pmb{\epsilon} \sim \mathcal{N}(\mathbf{0}, \mathbf{I})$ are independent. Our proof relies on the denoiser $d(\mathbf{y}, t) := \E[X_0 | X_t = \mathbf{y}]$ which is related to the score through Tweedie's identity,
\begin{equation*}
    s(\mathbf{y}, t) = \frac{r_t\, d(\mathbf{y}, t) - \mathbf{y}}{r_t^2 \sigma_t^2},
\end{equation*}
and its Jacobian is proportional to the posterior covariance,
\begin{equation*}
    \nabla_{\mathbf{y}} d(\mathbf{y}, t) = \tfrac{r_t}{r_t^2 \sigma_t^2} \text{Cov}(X_0 | X_t = \mathbf{y}) \succeq \mathbf{0}.
\end{equation*}

Using Tweedie's identity, we can write
\begin{align*}
    f(u) - \hat{I}_{\Delta}^{(2)} = \phi(u) - \phi_{\texttt{mid}} - \frac{u - 1/2}{r_t^2 \sigma_t^2}\|\widetilde{\mathbf{x}} - \mathbf{x}\|^2,
\end{align*}
where $\phi(u) = \frac{r_t}{r_t^2 \sigma_t^2} \langle d(\mathbf{x} + u(\widetilde{\mathbf{x}} - \mathbf{x}), t), \widetilde{\mathbf{x}} - \mathbf{x} \rangle$ and $\phi_{\texttt{mid}} = \frac{1}{2}(\phi(0) + \phi(1))$. Since the Jacobian of the denoiser is positive semi-definite, we have $\phi'(u) \geq 0$, and so $\phi(0) \leq \phi(u) \leq \phi(1)$. It follows that
\begin{align*}
    |\phi(u) - \phi_{\texttt{mid}}| &\leq \tfrac{1}{2}(\phi(1) - \phi(0)) \\
    &= \tfrac{r_t}{2r_t^2 \sigma_t^2} \big \langle d(\widetilde{\mathbf{x}}, t) - d(\mathbf{x}, t), \widetilde{\mathbf{x}} - \mathbf{x} \big \rangle\\
    &= \tfrac{1}{2} \langle s(\widetilde{\mathbf{x}}, t) - s(\mathbf{x}, t), \widetilde{\mathbf{x}} - \mathbf{x} \rangle + \tfrac{1}{2 r_t^2 \sigma_t^2}\|\widetilde{\mathbf{x}} - \mathbf{x}\|^2,
\end{align*}
where the last line uses Tweedie's identity once more. Finally, since $|u - 1/2| \leq 1/2$, the triangle inequality gives
\begin{align*}
    |f(u) - \hat{I}_{\Delta}^{(2)}| &\leq |\phi(u) - \phi_{\texttt{mid}}| + \tfrac{1}{2 r_t^2 \sigma_t^2}\|\widetilde{\mathbf{x}} - \mathbf{x}\|^2 \\
    &\leq \tfrac{1}{2} \langle s(\widetilde{\mathbf{x}}, t) - s(\mathbf{x}, t), \widetilde{\mathbf{x}} - \mathbf{x} \rangle + \tfrac{1}{r_t^2 \sigma_t^2}\|\widetilde{\mathbf{x}} - \mathbf{x}\|^2.
\end{align*}
Since this holds for every $u \in [0, 1]$, Assumption~\ref{ass:score_bound} holds with the stated choice of $C$.
\qed

\subsection{Proof of Lemma \ref{lem:poisson_density_ratio_estimator}}
\label{appendix:proof_poisson_density_ratio_estimator}
Recall that $|I_i - \hat{I}_{\Delta}^{(2)}| \leq C$ holds almost surely, for all $i$, and hence,
\begin{equation*}
    \frac12+\frac{I_i - \hat{I}_{\Delta}^{(2)}}{2C}\in[0,1].
\end{equation*}
Since $W$ is formed of a product of these variables, we also obtain $W \in [0, 1]$ almost surely.

Since $\E[I_i - \hat{I}_{\Delta}^{(2)}] = \log r - \hat{I}_{\Delta}^{(2)}$, and $\{I_i - \hat{I}_{\Delta}^{(2)}\}_{i=1}^N$ are i.i.d. and also independent of $N$, we obtain,
\begin{align*}
    \E[W|N] &= \prod_{i=1}^N \E \bigg [ \frac{1}{2} + \frac{I_i - \hat{I}_{\Delta}^{(2)}}{2C} \bigg ]\\
    &=
    \bigg ( \frac{1}{2} + \frac{\log r - \hat{I}_{\Delta}^{(2)}}{2C} \bigg )^N,
\end{align*}
where we use the convention $0^0 = 1$. Recalling that $N \sim \text{Poisson}(2 C)$, the law of total expectation gives,
\begin{align*}
    \E[W] &= \sum_{n=0}^\infty \frac{(2 C)^n \exp(-2 C)}{n!} \E[W|N=n] \\
    &= \exp(-2 C) \sum_{n=0}^\infty \frac{1}{n!} \big ( C + \log r - \hat{I}_{\Delta}^{(2)} \big )^n.
\end{align*}
Finally, recognising the sum as the Taylor series of the exponential, we obtain,
\begin{align*}
    \E[W] &= \exp(-2 C) \exp(C + \log r - \hat{I}_{\Delta}^{(2)})\\
    &= \exp(- C) \exp(\log r - \hat{I}_{\Delta}^{(2)})\\
    &= \exp(-[\hat{I}_{\Delta}^{(2)} + C(\mathbf{x},\widetilde{\mathbf{x}})])r_{p_t}(\mathbf{x}, \widetilde{\mathbf{x}})
\end{align*}
completing the proof.
\qed

\subsection{Proof of Theorem \ref{thm:two_coin_barker}}\label{proof:two_coin_barker}
In each outer iteration of Algorithm \ref{algorithm:exact_two_coin}, one of three events occurs. The iteration could produce an immediate rejection, which results from a coin flip with probability $\alpha'$. The iteration could also produce an acceptance, which requires first avoiding rejection and then winning the following acceptance coin flip with probability $W$. Since the immediate rejection flip and the acceptance flip are independent, the tower property gives this probability as
\begin{align*}
    (1 - \alpha') \E[W] = (1 - \alpha') \exp(-[\hat{I}_{\Delta}^{(2)} + C(\mathbf{x},\widetilde{\mathbf{x}})])r.
\end{align*}
Finally, it could occur that neither coin wins and neither a rejection or acceptance is returned, resulting in the loop restarting. This occurs with the remaining probability,
\begin{equation*}
    1 - \alpha' - (1 - \alpha') \exp(-[\hat{I}_{\Delta}^{(2)} + C(\mathbf{x},\widetilde{\mathbf{x}})])r = (1 - \alpha')(1 - \exp(-[\hat{I}_{\Delta}^{(2)} + C(\mathbf{x},\widetilde{\mathbf{x}})])r).
\end{equation*}

We begin by briefly analysing the number of iterations (of the outer loop) before termination, which we denote by $K$. Since $K$ being at least as large as $n$ is equivalent to the event that the first $n-1$ iterations produced restarts, it follows that,
\begin{align*}
    \prob(K \geq n) &= (1 - \alpha')^{n-1} (1 - \exp(-[\hat{I}_{\Delta}^{(2)} + C(\mathbf{x},\widetilde{\mathbf{x}})])r)^{n-1}.
\end{align*}
Since $H > 0$, we have $\alpha' < 1$ and so the restart probability satisfies
\begin{align*}
    (1 - \alpha') (1 - \exp(-[\hat{I}_{\Delta}^{(2)} + C(\mathbf{x},\widetilde{\mathbf{x}})])r) \in [0, 1).
\end{align*}
Hence, we have $\lim_{n \to \infty} \prob(K \geq n) = 0$. To show that $K$ is almost-surely finite, we note that for any $n$,
\begin{align*}
    \prob(K < \infty) \geq \prob(K < n) = 1 - \prob(K \geq n).
\end{align*}
In particular, we have that
\begin{equation*}
    \prob(K < \infty) \geq 1 - \lim_{n \to \infty} \prob(K \geq n) = 1.
\end{equation*}

We now compute the probability of acceptance. Let $A_n$ denote the events that the algorithm accepts on the $n$-th outer iteration. Since $A_n$ necessarily requires a sequence of $n-1$ no decision iterations followed by an acceptance, it follows directly that,
\begin{equation*}
    \Pr( A_n ) = (1 - \alpha')^{n-1} (1 - \exp(-[\hat{I}_{\Delta}^{(2)} + C(\mathbf{x},\widetilde{\mathbf{x}})])r)^{n-1} (1 - \alpha') \exp(-[\hat{I}_{\Delta}^{(2)} + C(\mathbf{x},\widetilde{\mathbf{x}})])r.
\end{equation*}
To compute the probability that the algorithm terminates with an acceptance, we must compute the probability of the union $\cup_{n=1}^\infty A_n$. Since the events $\{A_n\}_{n=1}^\infty$ are disjoint, we obtain the acceptance probability as a geometric series:
\begin{align*}
    \Pr\bigg ( \bigcup_{n=1}^\infty A_n \bigg ) &= \sum_{n=1}^{\infty} \Pr( A_n ) \\
    &= (1 - \alpha') \exp(-[\hat{I}_{\Delta}^{(2)} + C(\mathbf{x},\widetilde{\mathbf{x}})])r \sum_{n=1}^{\infty} (1 - \alpha')^{n-1} (1 - \exp(-[\hat{I}_{\Delta}^{(2)} + C(\mathbf{x},\widetilde{\mathbf{x}})])r)^{n-1}.
\end{align*}
Since the sum is a geometric series with rate $(1 - \alpha') (1 - \exp(-[\hat{I}_{\Delta}^{(2)} + C(\mathbf{x},\widetilde{\mathbf{x}})])r) \in [0, 1)$, it follows that the sum converges with,
\begin{align*}
    \Pr\bigg ( \bigcup_{n=1}^\infty A_n \bigg ) &= (1 - \alpha') \exp(-[\hat{I}_{\Delta}^{(2)} + C(\mathbf{x},\widetilde{\mathbf{x}})])r \frac{1}{1 - (1 - \alpha')(1 - \exp(-[\hat{I}_{\Delta}^{(2)} + C(\mathbf{x},\widetilde{\mathbf{x}})])r)}\\
    &= \frac{(1 - \alpha') \exp(-[\hat{I}_{\Delta}^{(2)} + C(\mathbf{x},\widetilde{\mathbf{x}})])r}{\alpha' + (1 - \alpha') \exp(-[\hat{I}_{\Delta}^{(2)} + C(\mathbf{x},\widetilde{\mathbf{x}})])r}.
\end{align*}
Recalling that $\alpha' = (1+He^{\hat{I}_{\Delta}^{(2)} + C})^{-1}$, this gives,
\begin{align*}
    \Pr\bigg ( \bigcup_{n=1}^\infty A_n \bigg ) &= \frac{He^{\hat{I}_{\Delta}^{(2)} + C} \exp(-[\hat{I}_{\Delta}^{(2)} + C]) r}{1 + He^{\hat{I}_{\Delta}^{(2)} + C} \exp(-[\hat{I}_{\Delta}^{(2)} + C]) r}\\
    &= \frac{Hr}{1+Hr}\\
    &= \alpha_B(\mathbf{x}, \tilde{\mathbf{x}}).
\end{align*}

Reversibility with respect to $p_t$ follows since the Barker acceptance function $a_B(r) = r/(1+r)$ satisfies $a_B(r) = r a_B(1/r)$ for all $r > 0$, which is a sufficient condition for detailed balance \citep[see, e.g.,][]{agrawal2023optimal}.

For convergence, we use that $p_t$ is a positive density and that the Langevin proposal has full support on $\mathbb{R}^d$, which guarantees $\alpha_B(\mathbf{x}, \tilde{\mathbf{x}}) > 0$ everywhere and hence that the chain is $\phi$-irreducible with respect to Lebesgue measure. Then Algorithm~\ref{algorithm:mala_ddm}  is Harris recurrent, which follows from the same line of argument as in \cite{Gareth-O-Roberts2006-mi}. Since $p_t$ is an invariant probability measure, the chain is positive recurrent. Convergence in total variation distance to $p_t$ follows.

\qed

\subsection{Proof of Proposition \ref{prop:two_coin_cost}}\label{proof:two_coin_cost}
From the proof of Theorem~\ref{thm:two_coin_barker}, the probability of termination in a single outer iteration is,
\begin{equation*}
    \alpha' + (1-\alpha')e^{-[\hat{I}_{\Delta}^{(2)} + C]}r = \frac{1}{1+He^{\hat{I}_{\Delta}^{(2)} + C}} + \frac{Hr}{1+He^{\hat{I}_{\Delta}^{(2)} + C}} = \frac{1+Hr}{1+He^{\hat{I}_{\Delta}^{(2)} + C}},
\end{equation*}
and since iterations are independent, the number of iterations $K$ before termination is geometric with this probability, establishing the first claim.

For the second claim, we reframe the setting as an optional stopping problem. Let $N_i$ denote the number of score queries at iteration $i$ where we set $N_i = 0$ if a rejection is returned at that iteration. With this, we obtain that the total number of score queries is given by $\E \bigg [ \sum_{i=1}^K N_i \bigg ]$.

For $i > K$, we continue to generate $N_i$ as if the algorithm has not terminated so that $(N_i)_{i=1}^\infty$ is an i.i.d. sequence. Let $\mathcal{F}_n$ denote the sigma-algebra generated by all randomness in iterations $1, \ldots, n$. From this, we see that $(N_i)_{i=1}^\infty$ is adapted to the filtration $(\mathcal{F}_i)_{i=1}^\infty$ and $K$ is a stopping time. With this, we describe the total number of queries as an optionally stopped sum over an i.i.d. process, and hence we can apply Wald's lemma. Provided $\E[N_1], \E[K] < \infty$, this lemma gives,
\begin{equation*}
    \E \bigg [ \sum_{i=1}^K N_i \bigg ] = \E[N_1] \E[K].
\end{equation*}
Since $N_1$ takes the value $0$ with probability $\alpha'$ and is $\mathrm{Poisson}(2C)$ distributed with probability $1-\alpha'$, we obtain the expectation
\begin{equation*}
    \E[N_1] = 2 (1 - \alpha') C = \frac{2CHe^{\hat{I}_{\Delta}^{(2)} + C}}{1+He^{\hat{I}_{\Delta}^{(2)} + C}}.
\end{equation*}
Since $K$ is geometrically distributed with parameter $(1+Hr)/(1+He^{\hat{I}_{\Delta}^{(2)} + C})$, we also obtain,
\begin{equation*}
    \E[K] = \frac{1+He^{\hat{I}_{\Delta}^{(2)} + C}}{1+Hr}.
\end{equation*}
Thus, we obtain the expected number of queries from the product,
\begin{equation*}
    \E \bigg [ \sum_{i=1}^K N_i \bigg ] = \frac{2CHe^{\hat{I}_{\Delta}^{(2)} + C}}{1+He^{\hat{I}_{\Delta}^{(2)} + C}} \cdot \frac{1+He^{\hat{I}_{\Delta}^{(2)} + C}}{1+Hr} = \frac{2CHe^{\hat{I}_{\Delta}^{(2)} + C}}{1+Hr}.
\end{equation*}

For the final claim, let $\delta \in (0, 1)$ and let $k$ be the smallest integer such that $k \geq \frac{1+He^{\hat{I}_{\Delta}^{(2)} + C}}{1+Hr} \log(2/\delta)$. Thus, for any $q \geq 0$, a union bound gives,
\begin{equation*}
    \prob \bigg ( \sum_{i=1}^K N_i > q \bigg ) \leq \prob(K > k) + \prob \bigg ( \sum_{i=1}^k N_i > q \bigg ).
\end{equation*}
Note that this bound does not require any independence between $K$ and $(N_i)_{i=1}^\infty$. For the first term, since $K$ is geometrically distributed, using $1 - x \leq e^{-x}$ and the choice of $k$, we obtain,
\begin{equation*}
    \prob(K > k) = \bigg ( 1 - \frac{1+Hr}{1+He^{\hat{I}_{\Delta}^{(2)} + C}} \bigg )^k \leq \exp \bigg ( - \frac{1+Hr}{1+He^{\hat{I}_{\Delta}^{(2)} + C}} k \bigg ) \leq \frac{\delta}{2}.
\end{equation*}
For the second term, we note that $N_1$ is either $0$ or $\mathrm{Poisson}(2C)$ distributed, so that $\E[e^{N_1}] \leq \exp(2C(e - 1))$. Hence, setting $q = 4Ck + \log(2/\delta)$, Markov's inequality and the independence of $(N_i)_{i=1}^\infty$ give,
\begin{align*}
    \prob \bigg ( \sum_{i=1}^k N_i > q \bigg ) &\leq e^{-q} \E[e^{N_1}]^k \\
    &\leq \exp \big ( 2C(e - 1)k - 4Ck - \log(2/\delta) \big ) \\
    &\leq \frac{\delta}{2}.
\end{align*}
Combining these bounds and using that $k \leq \frac{1+He^{\hat{I}_{\Delta}^{(2)} + C}}{1+Hr} \log(2/\delta) + 1$, we obtain that, with probability at least $1 - \delta$,
\begin{equation*}
    \sum_{i=1}^K N_i \leq q \leq 4C \frac{1+He^{\hat{I}_{\Delta}^{(2)} + C}}{1+Hr} \log(2/\delta) + 4C + \log(2/\delta).
\end{equation*}
\qed

\section{Results for Newton--Cotes quadrature rules}
\label{app:error_bounds}
\subsection{Other common quadrature rules}
\begin{example}[Simpson's 3/8 rule]
\label{example:simpsons_38_rule}
\normalfont
Setting $N=3$, $w_0=w_3=\frac{1}{8}$ and $w_1=w_2=\frac{3}{8}$ yields the estimate
\begin{align*}
    \hat{I}_{\Delta}^{(4)} = \frac{1}{8} \langle s(\mathbf{x}, t) + 3s(\mathbf{y}, t) + 3s(\mathbf{z}, t) + s(\tilde{\mathbf{x}},t), \tilde{\mathbf{x}} - \mathbf{x} \rangle, \quad \mathbf{y} = \mathbf{x} + \frac{1}{3}(\tilde{\mathbf{x}} - \mathbf{x}), \quad \mathbf{z} = \mathbf{x} + \frac{2}{3}(\tilde{\mathbf{x}} - \mathbf{x}).
\end{align*}
This has truncation error $-\frac{1}{6480} f^{(4)}(u^{*})$ for some $u^{*} \in [0,1]$, and is of order $\mathcal{O}(h^{5/2})$.
\end{example}

\subsection{Derivation of error bounds from truncation error}
Let $\mathbf{v} = \tilde{\mathbf{x}} - \mathbf{x}$. The results in this section are based on the following lemma.
\begin{lemma}
\label{lemma:proposal_distance_bound}
Let $\tilde{\mathbf{x}} = \mathbf{x} + \frac{h}{2} s(\mathbf{x}, t) + \sqrt{h} \mathbf{z}$ where $\mathbf{z} \sim \mathcal{N}(\mathbf{0}, \mathbf{I})$. Suppose $s(\mathbf{x}, t)$ is locally bounded. Then $\|\tilde{\mathbf{x}} -\mathbf{x}\|$ has $\mathcal{O}(h^{1/2})$.
\end{lemma}
\begin{proof}
By triangle inequality, we have 
\begin{align*}
    \|\mathbf{v}\| = \|\tilde{\mathbf{x}} - \mathbf{x}\| \leq  \frac{h}{2} \|s(\mathbf{x}, t)\| + \sqrt{h} \|\mathbf{z}\|.
\end{align*}
Since $s(\mathbf{x}, t)$ is locally bounded, the first term on the RHS of the inequality has order $\mathcal{O}(h)$. The second term on the RHS of the inequality has order $\mathcal{O}(h^{1/2})$. Since $h^{1/2}$ dominates $h$ as $h \to 0$, it follows that $\|\mathbf{v}\|$ has order $\mathcal{O}(h^{1/2})$.
\end{proof}

Recall from Section \ref{sec:quadrature_approximations} that  
\begin{align*}
    f:[0,1] \to \R, \quad f(u) = \big \langle s \big ( \mathbf{x} + u\mathbf{v}, t \big ), \mathbf{v} \big \rangle
\end{align*}
where $f$ is the integrand of Equation \eqref{eq:exact_i}. Truncation errors from the Newton--Cotes quadrature rules introduced in Section \ref{sec:quadrature_approximations} are expressed by derivatives of $f$. More specifically, they can be expressed by Fréchet derivatives of $s$ evaluated at $\mathbf{x} + u^{*} \mathbf{v}$. The $n$-th Fréchet derivative $D^ns(\mathbf{x}) : \underbrace{\mathbb{R}^d \times \dots \times \mathbb{R}^d}_{\text{$n$ times}} \to \mathbb{R}^d$ is a multi-linear map takes $n$ arguments from $\mathbb{R}^d$ and outputs a value from $\mathbb{R}^d$.

\paragraph{Trapezoidal rule.}
Recall from Example \ref{example:trapezoidal_rule} that the truncation error of the trapezoidal rule is $-\frac{1}{12} f^{\prime \prime}(u^{*})$ for some $u^{*} \in [0,1]$. This gives us
\begin{align*}
    |f^{\prime \prime}(u^{*})| &= |\langle D^2s(\mathbf{x} + u^{*} \mathbf{v})[\mathbf{v}, \mathbf{v}], \mathbf{v} \rangle |\\
    &\leq \|D^2s(\mathbf{x} + u^{*} \mathbf{v})[\mathbf{v}, \mathbf{v}] \| \|\mathbf{v}\|\\
    &\leq \|D^2s(\mathbf{x} + u^{*} \mathbf{v})\|_{\text{op}} \|\mathbf{v}\|^{3}
\end{align*}
where the first inequality follows from Cauchy-Schwarz and the second follows from the definition of operator norms for multi-linear maps. By Lemma \ref{lemma:proposal_distance_bound}, it follows that $|f^{\prime \prime}(u^{*})|$  has $\mathcal{O}(h^{3/2})$.

\paragraph{Simpson's 1/3 and 3/8 rules.}
Recall from Examples \ref{example:simpsons_13_rule} and \ref{example:simpsons_38_rule} that the truncation error of both Simpson's 1/3 and 3/8 rules are proportional $f^{(4)}(u^{*})$ for some $u^{*} \in [0,1]$. Following the same arguments as for the trapezoidal rule, this gives us
\begin{align*}
    |f^{(4)}(u^{*})| &= |\langle D^4s(\mathbf{x} + u^{*} \mathbf{v})[\mathbf{v}, \mathbf{v}, \mathbf{v}, \mathbf{v}], \mathbf{v} \rangle |\\
    &\leq \|D^4s(\mathbf{x} + u^{*} \mathbf{v})[\mathbf{v}, \mathbf{v}, \mathbf{v}, \mathbf{v}] \| \|\mathbf{v}\|\\
    &\leq \|D^4s(\mathbf{x} + u^{*} \mathbf{v})\|_{\text{op}} \|\mathbf{v}\|^{5}.
\end{align*}
By Lemma \ref{lemma:proposal_distance_bound}, it follows that $|f^{(4)}(u^{*})|$  has $\mathcal{O}(h^{5/2})$.

\section{Choosing the quantity \texorpdfstring{$C$}{C}}
\label{appendix:implementation_considerations_c}
The validity of Algorithm \ref{algorithm:exact_two_coin} rests on a valid estimate of $C(\mathbf{x}, \widetilde{\mathbf{x}})$ and, as shown in Proposition \ref{prop:two_coin_cost}, the computational cost depends on the tightness of this quantity. Choosing a suitable $C(\mathbf{x}, \tilde{\mathbf{x}})$ requires some additional knowledge about the score function as it bounds the score function along some, typically small, line segment. Thus, we require some predictability of how the score function changes along this line segment via smoothness or boundedness of the score network.

\paragraph{Bounded data support.}
Suppose that the data distribution is known to be supported on the centred ball of radius $b$. By Tweedie's identity, the score admits the form,
\begin{equation*}
    \nabla \log p_t = \frac{r_t\, d(\mathbf{x}, t) - \mathbf{x}}{r_t^2 \sigma_t^2}, \qquad d(\mathbf{x}, t) := \E[X_0|X_t=\mathbf{x}].
\end{equation*}
Thus, the score function can be controlled using the fact that $\|d(\mathbf{x}, t)\| \leq b$. Similarly, given a generic score function $s(\mathbf{x}, t)$ that is, for example, derived from a trained neural network, it may be parameterised such that the associated denoising function $\hat{d}(\mathbf{x}, t)$ is appropriately compact. Indeed, if we assume that $\|\hat{d}(\cdot, t)\| \leq b$ for some $b > 0$, it follows that for any $\mathbf{x}, \widetilde{\mathbf{x}} \in \R^d$ and $u \in [0, 1]$,
\begin{align*}
    \left|\left\langle
    s(\mathbf{x}+u(\widetilde{\mathbf{x}}-\mathbf{x}),t),
    \widetilde{\mathbf{x}}-\mathbf{x}
    \right\rangle\right| &= \frac{1}{r_t^2 \sigma_t^2} |r_t \langle d(\mathbf{x}+u(\widetilde{\mathbf{x}}-\mathbf{x}),t),
    \widetilde{\mathbf{x}}-\mathbf{x} \rangle - \langle \mathbf{x}+u(\widetilde{\mathbf{x}}-\mathbf{x}),
    \widetilde{\mathbf{x}}-\mathbf{x} \rangle|\\
    &\leq \frac{1}{r_t \sigma_t^2} \| d(\mathbf{x}+u(\widetilde{\mathbf{x}}-\mathbf{x}),t)\| \|\widetilde{\mathbf{x}}-\mathbf{x} \| + \frac{1}{r_t^2 \sigma_t^2} |\langle \mathbf{x}+u(\widetilde{\mathbf{x}}-\mathbf{x}),
    \widetilde{\mathbf{x}}-\mathbf{x} \rangle|\\
    &\leq \frac{br_t + \max\{\| \mathbf{x}\|, \|\widetilde{\mathbf{x}}\|\}}{r_t^2 \sigma_t^2} \|\widetilde{\mathbf{x}}-\mathbf{x} \|.
\end{align*}
Since this holds for every $u \in [0, 1]$, it also bounds $|\hat{I}_{\Delta}^{(2)}|$, and thus setting $C(\mathbf{x}, \widetilde{\mathbf{x}})$ to be twice the right-hand side verifies Assumption \ref{ass:score_bound}.

\paragraph{Lipschitz score.}
The previous approach relied on boundedness in the score. Here we utilise first-order smoothness, by supposing that the score function $s(\mathbf{x}, t)$ is Lipschitz in $\mathbf{x}$. This means that there exists some constant $L \geq 0$ such that,
\begin{equation*}
    \|s(\mathbf{x}, t) - s(\widetilde{\mathbf{x}}, t)\| \leq L \|\mathbf{x} - \widetilde{\mathbf{x}}\|,
\end{equation*}
for all $\mathbf{x}, \widetilde{\mathbf{x}}$. As in Section~\ref{sec:quadrature_approximations}, let $f(u) = \langle s(\mathbf{x}+u(\widetilde{\mathbf{x}}-\mathbf{x}),t), \widetilde{\mathbf{x}}-\mathbf{x} \rangle$, so that $\hat{I}_{\Delta}^{(2)} = \frac{1}{2}(f(0) + f(1))$. By the Cauchy--Schwarz inequality, we have $|f(u) - f(0)| \leq L u \|\widetilde{\mathbf{x}}-\mathbf{x}\|^2$ and $|f(u) - f(1)| \leq L (1 - u) \|\widetilde{\mathbf{x}}-\mathbf{x}\|^2$ for $u \in [0, 1]$. Hence,
\begin{equation*}
    \big | f(u) - \hat{I}_{\Delta}^{(2)} \big | \leq \tfrac{1}{2} |f(u) - f(0)| + \tfrac{1}{2} |f(u) - f(1)| \leq \tfrac{L}{2} \|\widetilde{\mathbf{x}}-\mathbf{x}\|^2.
\end{equation*}
Thus, Assumption \ref{ass:score_bound} holds with constant,
\begin{equation*}
    C(\mathbf{x}, \widetilde{\mathbf{x}}) = \frac{L}{2} \|\widetilde{\mathbf{x}}-\mathbf{x}\|^2.
\end{equation*}
Note that in fact, this proof only requires the weaker condition given by,
\begin{equation*}
    |\langle s(\mathbf{x}, t) - s(\widetilde{\mathbf{x}}, t), \mathbf{x} - \widetilde{\mathbf{x}}\rangle| \leq L \|\mathbf{x} - \widetilde{\mathbf{x}}\|^2.
\end{equation*}

\section{Implementation considerations}
\subsection{Practical termination}
\label{sec:madm_hybrid}
\begin{figure}[t]
\begin{algorithm}[H]
   \caption{Hybrid MADM algorithm}
   \label{algorithm:combined_accept_reject}
\begin{algorithmic}
    \State {\bfseries Input:} $\mathbf{x}$, $\widetilde{\mathbf{x}}$, $s(\cdot, t)$, $h$, $C(\cdot, \cdot)$, $\texttt{use\_exact} \in \{\texttt{True}, \texttt{False}\}$.
    \State Set $\alpha' = (1 + H_t(\mathbf{x}, \widetilde{\mathbf{x}}) \exp(\hat{I}_{\Delta}^{(2)} + C(\mathbf{x}, \widetilde{\mathbf{x}})))^{-1}$
    \If{$\texttt{use\_exact}$}
        \For{$k=0,\ldots,K-1$}
            \If{$U \leq \alpha'~\text{for}~U \sim \mathcal{U}([0,1])$}
                \State \Return REJECT.
            \EndIf
            \State Compute $W$.
            \If{$U \leq W~\text{for}~U \sim \mathcal{U}([0,1])$}
                \State \Return ACCEPT.
            \EndIf
        \EndFor
    \EndIf
        
    \State \Return Alg.~\ref{algorithm:simpsons_adjustment}
\end{algorithmic}
\end{algorithm}
\end{figure}

Algorithm \ref{algorithm:exact_two_coin} gives an exact method and avoids discretisation bias, but may require many score evaluations. Algorithm \ref{algorithm:simpsons_adjustment} is biased since it perturbs the acceptance probability, but it is computationally much cheaper. To get the best of both worlds, we propose a hybrid sampling algorithm: attempt Algorithm \ref{algorithm:exact_two_coin} $K$ times to get an accept or reject decision; use Algorithm \ref{algorithm:simpsons_adjustment} if this fails. This controls the worst-case computational cost while retaining an exact update whenever the Bernoulli factory terminates in time. In our experiments, we used this hybrid sampling strategy for synthetic 2D datasets, where it is feasible, and showed that on image datasets, using only Algorithm \ref{algorithm:simpsons_adjustment} still yields performance gains.

\subsection{Optimal scaling}\label{app:proof_optimal_scaling}


Optimal scaling theory addresses the choice of step size $h$ to achieve efficient exploration of the target: if proposed moves are too small, the chain mixes slowly, whereas overly large proposals lead to frequent rejections and poor exploration. The goal is to strike an optimal balance by selecting a proposal step size $h$ that maximizes sampling efficiency. We establish the optimal scaling of $h$ in Algorithm~\ref{algorithm:exact_two_coin} by maximizing the Expected Squared Jumping Distance (ESJD) in the stationary phase:
\[
\max_h\text{ESJD}(h):=\max_h\E_{\mathbf{x}_t^{(i)}\sim p_t}[\|\mathbf{x}_t^{(i+1)}-\mathbf{x}_t^{(i)}\|^2].
\]
The justification for using ESJD as an optimization criterion comes from its connection to diffusion limits \citep{roberts97}. When weak convergence to a diffusion limit is established, the ESJD of the Markov chain converges to the quadratic variation of the limiting diffusion \citep{roberts01,Yang2020}. A larger ESJD generally indicates that the chain is exploring the state space more efficiently, avoiding both overly conservative moves and excessive rejection.

We establish the optimal scaling theory for the Barker-adjusted Langevin algorithm in Algorithm~\ref{algorithm:exact_two_coin}. For clarity, we present the result for multivariate Gaussian targets. The argument extends in a standard way to broader classes of target distributions \citep{Yang2020}, with the same limiting optimal acceptance rate.

\begin{proposition}\label{thm_optimal_scaling}
   Assume that \(p_t\) is a \(d\)-dimensional standard multivariate Gaussian and that the score is exact, i.e., \(s(\cdot,t)=\nabla \log p_t(\cdot)\). Choosing \(h\) to maximize the ESJD for Algorithm~\ref{algorithm:exact_two_coin}, then the optimal scaling for $h$ is $h=h_d\sim d^{-1/3}$, and the optimal expected acceptance rate converges to \emph{0.347} (to three decimal places) as \(d \to \infty\).
\end{proposition}
Note that if $h$ scales with $d$ faster (or slower) than order $d^{-1/3}$, the limiting acceptance probability becomes $1/2$ (or $0$), which cannot be optimal for maximizing the ESJD. Therefore, we will write $h=h_d=\ell^2d^{-1/3}$ for some $\ell>0$. The proof follows the same line of argument as in \cite{roberts98}, with the acceptance function modified according to \cite{agrawal2023optimal}. 

First, for notational clarity, we suppress the time index in the target $p_t$. By assumption,
\[
p_t(\mathbf{x}) = \prod_{i=1}^d \phi(x_i), \qquad \text{where } \phi(x_i) = (2\pi)^{-1/2} \exp(-x_i^2/2).
\]
Under a perfect score, the resulting exact algorithm reduces to the Barker-adjusted Langevin algorithm: proposals are generated via ULA and accepted according to Barker's rule. For consistency in notation, we also suppress the subscript $B$ in the acceptance probability $\alpha_B$ and write $\alpha_d$ to highlight its dependence on the dimension. That is,
\begin{equation}
    \alpha_d(\mathbf{x}, \tilde{\mathbf{x}}) = \frac{1}{1 + \exp(-W_d)}
\end{equation}
where $W_d = \log\left(\frac{p_t(\tilde{\mathbf{x}})q(\mathbf{x}\mid\tilde{\mathbf{x}})}{p_t(\mathbf{x})q(\tilde{\mathbf{x}}\mid \mathbf{x})}\right)$ is the log-acceptance ratio, and the ULA proposal
\begin{equation}
    \tilde{x}_i = x_i - \frac{h}{2} x_i + \sqrt{h} z_i, \quad i = 1, \dots, d,
\end{equation}
where $x_i, z_i \stackrel{iid}{\sim} \mathcal{N}(0, 1)$ in stationarity. Note that while the assumption of a multivariate Gaussian target simplifies the analysis of ULA proposal, the argument extends readily to more general settings \citep{Yang2020}. We restrict attention to the Gaussian case only for simplicity.

We are interested in maximizing the ESJD, with step size scales with dimension as $h = \ell^2 d^{-1/3}$ for a constant $\ell > 0$ to yield a non-degenerate limit:
\begin{equation}
    \textrm{ESJD}_d(h) = \E[\|\mathbf{x}-\tilde{\mathbf{x}}\|^2\alpha_d(\mathbf{x},\tilde{\mathbf{x}})]= d\cdot \E\left[ (\tilde{x}_1 - x_1)^2 \alpha_d(\mathbf{x}, \tilde{\mathbf{x}}) \right]
\end{equation}
As we re-parameterize $h$ by $\ell$ via $h=\ell^2 d^{-1/3}$, we can define an efficiency function over $\ell$:
\begin{equation}
    J_d(\ell) := d^{-2/3} \cdot \text{ESJD}_d(h).
\end{equation}
Then maximizing $\text{ESJD}(h)$ over $h$ is equivalent to maximizing $J_d(\ell)$ over $\ell$ up to a constant $d^{-2/3}$.

Next, we establish the limiting distribution of $W_d$. As shown in standard algebraic expansions for Gaussian targets, the log-ratio simplifies perfectly to:
\begin{equation}
    W_d = \sum_{i=1}^d \xi_{i,d}, \quad \text{where} \quad \xi_{i,d} = \frac{h}{8} (x_i^2 - \tilde{x}_i^2)
\end{equation}

\begin{lemma}
As $d \to \infty$, $W_d \xrightarrow{\mathcal{D}} \mathcal{N}\left(-\frac{\ell^6}{32}, \frac{\ell^6}{16}\right)$.
\end{lemma}
\begin{proof}
The proof follows the same line of arguments as the proof for MALA in \cite{roberts98}. Substituting the ULA proposal into $\xi_{i,d}$, we obtain:
\begin{equation}
    \xi_{i,d} = \frac{h}{8} \left[ hx_i^2 - \frac{h^2}{4}x_i^2 - 2\sqrt{h}\left(1-\frac{h}{2}\right)x_i z_i - hz_i^2 \right].
\end{equation}
Taking expectations using $\mathbb{E}[x_i^2]=1, \mathbb{E}[z_i^2]=1$, and $\mathbb{E}[x_i z_i]=0$:
\begin{equation}
    \mu_d = \mathbb{E}[\xi_{i,d}] = \frac{h}{8} \left( h - \frac{h^2}{4} - h \right) = -\frac{h^3}{32}.
\end{equation}
For the variance, the dominant term is the cross-term $-\frac{h^{3/2}}{4} x_i z_i$. Thus:
\begin{equation}
    \mathrm{Var}(\xi_{i,d}) = \mathbb{E}\left[ \left(-\frac{h^{3/2}}{4} x_i z_i\right)^2 \right] + \mathcal{O}(h^4) = \frac{h^3}{16} + \mathcal{O}(h^4).
\end{equation}
Because $h = \ell^2 d^{-1/3}$, the sum of means is $\mathbb{E}[W_d] = d \left(-\frac{h^3}{32}\right) = -\frac{\ell^6}{32}$. 
The sum of variances is $s_d^2 = \sum_{i=1}^d \mathrm{Var}(\xi_{i,d}) = d\left(\frac{h^3}{16}\right) + d\mathcal{O}(h^4) \to \frac{\ell^6}{16}$. Finally, to establish asymptotic normality, we verify the Lyapunov condition for the fourth moment,
 \[
        \sum_{i=1}^d \E\left[|\xi_{i,d}-\E\xi_{i,d}|^4\right]
        =O(dh^6)=O(d^{-1})\to 0,
    \]
    which completes the proof.
\end{proof}

Next, we establish the scaling limit of ESJD and the expected acceptance rate.
\begin{proposition}[Optimal Scaling Limit]\label{thm_BALA}
As $d \to \infty$, the scaled ESJD converges to:
\begin{equation}
    \lim_{d \to \infty} J_d(\ell) = \ell^2\cdot A(\ell),\quad A(\ell):=\mathbb{E}\left[ \frac{1}{1 + \exp(-W)} \right]
\end{equation}
where $W \sim \mathcal{N}(-\sigma^2/2, \sigma^2)$ and $\sigma^2=\ell^6/16$.
\end{proposition}
    \begin{proof}
    The proof follows the same line of argument as in \cite{roberts98,agrawal2023optimal}.
Observe that in stationarity, the $d$ components are identically distributed. Thus, the full ESJD is exactly $d$ times the single-coordinate ESJD:
\begin{equation}
    J_d(\ell) = d^{-2/3} \cdot d \cdot \mathbb{E}\left[ (\tilde{x}_1 - x_1)^2 \alpha_d(\mathbf{x}, \tilde{\mathbf{x}}) \right] = d^{1/3} \mathbb{E}\left[ (\tilde{x}_1 - x_1)^2 \alpha_d(\mathbf{x}, \tilde{\mathbf{x}}) \right]
\end{equation}

Write the single-coordinate squared jump distance as $(\tilde{x}_1 - x_1)^2 = h z_1^2 + R_{1,d}$, where $R_{1,d} = \mathcal{O}(h^{3/2})$ almost surely. Multiplying by $d^{1/3}$, we note $d^{1/3} h = \ell^2$. Thus, $d^{1/3} (\tilde{x}_1 - x_1)^2 \to \ell^2 z_1^2$ almost surely. Let $W_{d,-1} = \sum_{i=2}^d \xi_{i,d}$. Because $\xi_{1,d} \xrightarrow{\mathbb{P}} 0$, by Slutsky's Theorem, $W_{d,-1} \xrightarrow{\mathcal{D}} W$. Furthermore, since $W_{d,-1}$ is strictly independent of $z_1$, the joint distribution converges: $(z_1, W_{d,-1}) \xrightarrow{\mathcal{D}} (z_1, W)$, where $z_1 \perp \!\!\! \perp W$. By Slutsky's Theorem again on the pair $(z_1, W_d) = (z_1, W_{d,-1}) + (0, \xi_{1,d})$, we have $(z_1, W_d) \xrightarrow{\mathcal{D}} (z_1, W)$.

Let $g(w) = (1 + \exp(-w))^{-1}$ represent Barker's function. Because $g$ is continuous and bounded, $g(w) \in (0,1)$, the function $f(z, w) = z^2 g(w)$ is continuous. By the Continuous Mapping Theorem:
\begin{equation}
    V_d:=d^{1/3} (\tilde{x}_1 - x_1)^2 \alpha_d(\mathbf{x}, \tilde{\mathbf{x}}) \xrightarrow{\mathcal{D}} \ell^2 z_1^2 g(W).
\end{equation}

To pass the limit inside the expectation, we require the sequence $V_d$ to be uniformly integrable. Indeed, since $0 \le \alpha_d \le 1$, we have $|V_d| \leq d^{1/3} [ \frac{h^2}{4}x_1^2 - h^{3/2}x_1 z_1 + h z_1^2 ] \leq \frac{\ell^4}{4}x_1^2 + \ell^3|x_1 z_1| + \ell^2 z_1^2$. Because $x_1, z_1 \sim \mathcal{N}(0,1)$, the dominating variable is integrable, which implies uniform integrability of $\{V_d\}_{d \ge 1}$.
It follows that:
\begin{equation}
    \lim_{d\to\infty} J_d(\ell) = \lim_{d\to\infty} \mathbb{E}[V_d] = \mathbb{E}\left[ \ell^2 z_1^2 g(W) \right] = \ell^2 \mathbb{E}[z_1^2] \mathbb{E}[g(W)] = \ell^2 \mathbb{E}\left[\frac{1}{1 + \exp(-W)}\right].
\end{equation}
\end{proof}

Finally, Proposition \ref{thm_BALA} allows one to determine the optimal acceptance rate via numerical simulations. See Figure \ref{fig:BALA} for the efficiency curve, $\ell^2\cdot A(\ell)$ versus $A(\ell)$. By choosing the optimal $\ell$, which corresponds to the optimal step size $h$ that maximizes the ESJD, the expected acceptance rate converges to $0.347$ (to three decimal places) as $d\to\infty$.

\begin{figure}
    \centering
    \includegraphics[width=0.5\linewidth]{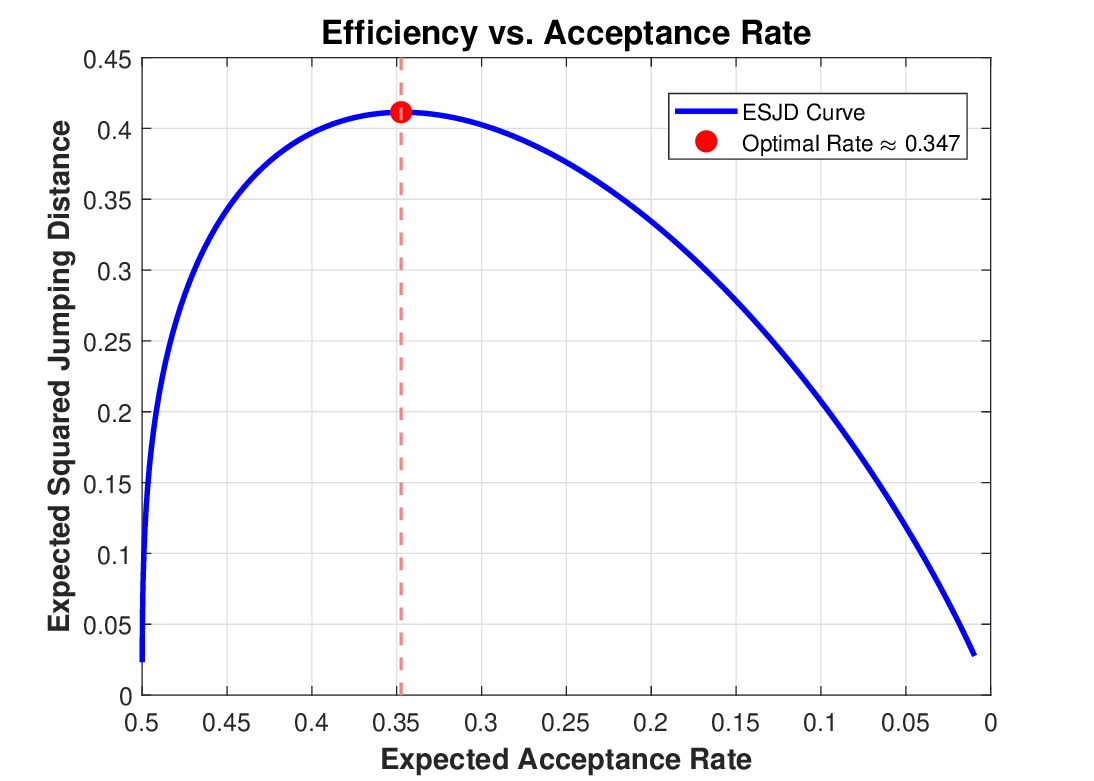}
    \caption{Optimal acceptance rate for Algorithm \ref{algorithm:exact_two_coin}}
    \label{fig:BALA}
\end{figure}

\section{Convergence analysis}
\subsection{Convergence of MADM under a warm start}\label{appendix:convergence_warm_start}

In this section, we establish convergence guarantees of the MADM corrector algorithm at a fixed noise level when initialised from a warm start. We begin with a general result for Barker-adjusted Markov chains.

As in Section~\ref{sec:mcmc}, we take $p$ to be a positive target density on $\R^d$ and $q$ to be a proposal density. We write $P_{\mathrm{B}}$ for the transition kernel of the Markov chain with proposal $q$ and acceptance probability $\alpha_{\mathrm{B}}$ given in \eqref{eq:acceptanceproba}. That is, for every $\mathbf{x} \in \R^d$ and measurable $A \subseteq \R^d$,
\begin{equation*}
    P_{\mathrm{B}}(\mathbf{x}, A) = \int_A \alpha_{\mathrm{B}}(\mathbf{x}, \widetilde{\mathbf{x}}) q(\widetilde{\mathbf{x}} \mid \mathbf{x}) \rmd \widetilde{\mathbf{x}} + \bigg ( 1 - \int_{\R^d} \alpha_{\mathrm{B}}(\mathbf{x}, \widetilde{\mathbf{x}}) q(\widetilde{\mathbf{x}} \mid \mathbf{x}) \rmd \widetilde{\mathbf{x}} \bigg ) \1_A(\mathbf{x}),
\end{equation*}
and we define $P_{\mathrm{MH}}$ analogously, with $\alpha_{\mathrm{B}}$ replaced by $\alpha_{\mathrm{MH}}$. We refer to these as the Barker-adjusted and Metropolis-adjusted kernels respectively. Both are reversible with respect to $p$, so that $p$ is a stationary distribution for each.

We use the notion of $s$-conductance introduced by \citet{Lovasz1993-wb}. For a transition kernel $P$ with stationary distribution $p$, the flow of $P$ out of a measurable set $A \subseteq \R^d$ is given by,
\begin{equation*}
    \Phi_P(A) = \int_A P(\mathbf{x}, A^c) p(\mathbf{x}) \rmd \mathbf{x}.
\end{equation*}
For $s \in (0, 1/2]$, the $s$-conductance of $P$ is then defined as,
\begin{equation*}
    \Phi_s(P) = \inf \bigg \{ \frac{\Phi_P(A)}{p(A) - s} : s < p(A) \leq \frac{1}{2} \bigg \}.
\end{equation*}

The following lemma is a modification of Corollary~1.6 in \citet{Lovasz1993-wb} refined for Barker-adjusted Markov chains. Note that, unlike the result of \citet{Lovasz1993-wb}, we do not rely on $1/2$-laziness. We find that such an assumption is not necessary for Barker-adjusted chains and that the resulting conductance quantity is that of the lazy Metropolis-adjusted chain.

\begin{lemma}
\label{lem:barker_conductance}
Let $(\mathbf{X}^{(i)})_{i \geq 0}$ be a Markov chain with transition kernel $P_{\mathrm{B}}$, let $s \in (0, 1/2]$ and let
\begin{equation*}
    M_s = \sup \big \{ |\prob(\mathbf{X}^{(0)} \in A) - p(A)| : p(A) \leq s \big \}.
\end{equation*}
Then, for every $i \geq 0$ and measurable $A \subseteq \R^d$,
\begin{equation*}
    |\prob(\mathbf{X}^{(i)} \in A) - p(A)| \leq M_s + \frac{M_s}{s} \bigg ( 1 - \frac{\Phi_s(P_{\mathrm{MH}})^2 \wedge 1}{8} \bigg )^i.
\end{equation*}
\end{lemma}
\begin{proof}
For the sake of brevity, we write $\mu_i$ for the law of $\mathbf{X}^{(i)}$. Following \citep[Section~1.C]{Lovasz1993-wb}, for $u \in [0, 1]$ we define,
\begin{equation*}
    h_i(u) = \sup_g \bigg \{ \int g \rmd \mu_i - u \bigg \},
\end{equation*}
where the supremum is over all measurable functions $g : \R^d \to [0, 1]$ such that $\int g(\mathbf{x}) p(\mathbf{x}) \rmd \mathbf{x} = u$.
\citet{Lovasz1993-wb} show that $h_i$ satisfies the following properties, none of which rely on laziness: (i) $h_i$ is concave on $[0, 1]$; (ii) for every $u \in [0, 1]$ the supremum is attained by the indicator function of a measurable set $A$ with $p(A) = u$; and (iii) $h_i(u)$ is non-increasing in $i$.

Fix $i \geq 1$, $s \leq u \leq 1/2$ and let $A$ be a measurable set such that $p(A) = u$ and $h_i(u) = \mu_i(A) - u$, which exists by (ii). For $\mathbf{x}, \widetilde{\mathbf{x}} \in \R^d$, let
\begin{equation*}
    \theta(\mathbf{x}, \widetilde{\mathbf{x}}) = \big ( 1 - \alpha_{\mathrm{B}}(\mathbf{x}, \widetilde{\mathbf{x}}) \big ) \1_A(\mathbf{x}) + \alpha_{\mathrm{B}}(\mathbf{x}, \widetilde{\mathbf{x}}) \1_A(\widetilde{\mathbf{x}}),
\end{equation*}
denote the probability that the chain lands in $A$ given the current state $\mathbf{x}$ and the proposal $\widetilde{\mathbf{x}}$, so that $P_{\mathrm{B}}(\mathbf{x}, A) = \int \theta(\mathbf{x}, \widetilde{\mathbf{x}}) q(\widetilde{\mathbf{x}} \mid \mathbf{x}) \rmd \widetilde{\mathbf{x}}$. We split this function as, $P_{\mathrm{B}}(\mathbf{x}, A) = \frac{1}{2}(g_1(\mathbf{x}) + g_2(\mathbf{x}))$ where,
\begin{align*}
    g_1(\mathbf{x}) &= \int \big ( (2\theta(\mathbf{x}, \widetilde{\mathbf{x}}) - 1) \vee 0 \big ) q(\widetilde{\mathbf{x}} \mid \mathbf{x}) \rmd \widetilde{\mathbf{x}}, \\
    g_2(\mathbf{x}) &= \int \big ( 2\theta(\mathbf{x}, \widetilde{\mathbf{x}}) \wedge 1 \big ) q(\widetilde{\mathbf{x}} \mid \mathbf{x}) \rmd \widetilde{\mathbf{x}},
\end{align*}
and we note that both functions take values in $[0, 1]$. For $j \in \{1, 2\}$, let $u_j = \int g_j(\mathbf{x}) p(\mathbf{x}) \rmd \mathbf{x}$ so that,
\begin{align*}
    u_1 + u_2 &= 2 \int P_{\mathrm{B}}(\mathbf{x}, A) p(\mathbf{x}) \rmd \mathbf{x}\\
    &= 2u.
\end{align*}
Since $g_j$ is measurable and takes values in $[0, 1]$, we obtain that
\begin{align*}
    h_i(u) &= \int P_{\mathrm{B}}(\mathbf{x}, A) \rmd \mu_{i-1}(\mathbf{x}) - u \\
    &= \frac{1}{2} \bigg ( \int g_1 \rmd \mu_{i-1} - u_1 \bigg ) + \frac{1}{2} \bigg ( \int g_2 \rmd \mu_{i-1} - u_2 \bigg ) \\
    &\leq \frac{1}{2} h_{i-1}(u_1) + \frac{1}{2} h_{i-1}(u_2).
\end{align*}
Next, we use that $(2\theta \wedge 1) - \theta = \theta \wedge (1 - \theta)$ to obtain,
\begin{align*}
    u_2 - u &= \int \big ( g_2(\mathbf{x}) - P_{\mathrm{B}}(\mathbf{x}, A) \big ) p(\mathbf{x}) \rmd \mathbf{x} \\
    &= \iint \big ( \theta(\mathbf{x}, \widetilde{\mathbf{x}}) \wedge (1 - \theta(\mathbf{x}, \widetilde{\mathbf{x}})) \big ) p(\mathbf{x}) q(\widetilde{\mathbf{x}} \mid \mathbf{x}) \rmd \widetilde{\mathbf{x}} \rmd \mathbf{x}\\
    &= \bigg ( \int_A \int_{A^c} + \int_{A^c} \int_A \bigg ) \big ( \alpha_{\mathrm{B}}(\mathbf{x}, \widetilde{\mathbf{x}}) \wedge (1 - \alpha_{\mathrm{B}}(\mathbf{x}, \widetilde{\mathbf{x}})) \big ) p(\mathbf{x}) q(\widetilde{\mathbf{x}} \mid \mathbf{x}) \rmd \widetilde{\mathbf{x}} \rmd \mathbf{x}.
\end{align*}
By \eqref{eq:acceptanceproba}, we have $\alpha_{\mathrm{B}} \wedge (1 - \alpha_{\mathrm{B}}) = \alpha_{\mathrm{MH}}/(1 + R)$. We exchange the roles of $\mathbf{x}$ and $\widetilde{\mathbf{x}}$ in the second integral and use that, by the definition of $R$ in \eqref{eq:ratiotargetprop},
\begin{align*}
    p(\widetilde{\mathbf{x}}) q(\mathbf{x} \mid \widetilde{\mathbf{x}}) &= R(\mathbf{x}, \widetilde{\mathbf{x}}) p(\mathbf{x}) q(\widetilde{\mathbf{x}} \mid \mathbf{x}), \qquad R(\widetilde{\mathbf{x}}, \mathbf{x}) = R(\mathbf{x}, \widetilde{\mathbf{x}})^{-1}, \\
    \alpha_{\mathrm{MH}}(\widetilde{\mathbf{x}}, \mathbf{x}) &= R(\mathbf{x}, \widetilde{\mathbf{x}})^{-1} \alpha_{\mathrm{MH}}(\mathbf{x}, \widetilde{\mathbf{x}}),
\end{align*}
to obtain,
\begin{align*}
    u_2 - u &= \int_A \int_{A^c} \bigg ( \frac{\alpha_{\mathrm{MH}}(\mathbf{x}, \widetilde{\mathbf{x}})}{1 + R(\mathbf{x}, \widetilde{\mathbf{x}})} + R(\mathbf{x}, \widetilde{\mathbf{x}}) \frac{\alpha_{\mathrm{MH}}(\widetilde{\mathbf{x}}, \mathbf{x})}{1 + R(\widetilde{\mathbf{x}}, \mathbf{x})} \bigg ) p(\mathbf{x}) q(\widetilde{\mathbf{x}} \mid \mathbf{x}) \rmd \widetilde{\mathbf{x}} \rmd \mathbf{x} \\
    &= \int_A \int_{A^c} \bigg ( \frac{\alpha_{\mathrm{MH}}(\mathbf{x}, \widetilde{\mathbf{x}})}{1 + R(\mathbf{x}, \widetilde{\mathbf{x}})} + R(\mathbf{x}, \widetilde{\mathbf{x}}) \frac{R(\mathbf{x}, \widetilde{\mathbf{x}})^{-1} \alpha_{\mathrm{MH}}(\mathbf{x}, \widetilde{\mathbf{x}})}{1 + R(\mathbf{x}, \widetilde{\mathbf{x}})^{-1}} \bigg ) p(\mathbf{x}) q(\widetilde{\mathbf{x}} \mid \mathbf{x}) \rmd \widetilde{\mathbf{x}} \rmd \mathbf{x} \\
    &= \int_A \int_{A^c} \alpha_{\mathrm{MH}}(\mathbf{x}, \widetilde{\mathbf{x}}) p(\mathbf{x}) q(\widetilde{\mathbf{x}} \mid \mathbf{x}) \rmd \widetilde{\mathbf{x}} \rmd \mathbf{x}.
\end{align*}
Thus, by the definition of $s$-conductance and using $p(A) = u \in [s, 1/2]$ (the case $u = s$ being trivial),
\begin{equation*}
    u_2 - u = \Phi_{P_{\mathrm{MH}}}(A) \geq \Phi_s(P_{\mathrm{MH}})(u - s).
\end{equation*}
Since $u_1 + u_2 = 2u$, we also obtain that $u - u_1 \geq \Phi_s(P_{\mathrm{MH}})(u - s)$. Thus, we have that $u_1 \leq u - (\Phi_s(P_{\mathrm{MH}}) \wedge 1)(u - s) \leq u + (\Phi_s(P_{\mathrm{MH}}) \wedge 1)(u - s) \leq u_2$ and hence, by the concavity of $h_{i-1}$,
\begin{equation*}
    h_i(u) \leq \frac{1}{2} \Big ( h_{i-1}\big(u - (\Phi_s(P_{\mathrm{MH}}) \wedge 1)(u - s)\big) + h_{i-1}\big(u + (\Phi_s(P_{\mathrm{MH}}) \wedge 1)(u - s)\big) \Big ).
\end{equation*}
When $1/2 \leq u \leq 1 - s$, the same argument applies, using that $\Phi_{P_{\mathrm{MH}}}(A) = \Phi_{P_{\mathrm{MH}}}(A^c)$ by reversibility and $p(A^c) = 1 - u$, which gives the analogous inequality,
\begin{equation*}
    h_i(u) \leq \frac{1}{2} \Big ( h_{i-1}\big(u - (\Phi_s(P_{\mathrm{MH}}) \wedge 1)(1 - u - s)\big) + h_{i-1}\big(u + (\Phi_s(P_{\mathrm{MH}}) \wedge 1)(1 - u - s)\big) \Big ).
\end{equation*}

An upper bound on $h_i$ then follows from the induction argument in Theorem~1.4 of \citet{Lovasz1993-wb}. Here, it is shown that for any $c_1, c_2$ such that,
\begin{equation*}
    h_0(u) \leq c_1 + c_2 \Big ( \sqrt{u - s} \wedge \sqrt{1 - s - u} \Big ), \qquad \text{ for all } s \leq u \leq 1 - s,
\end{equation*}
it follows that, for all $i \geq 0, s \leq u \leq 1 - s$,
\begin{equation*}
    h_i(u) \leq c_1 + c_2 \Big ( \sqrt{u - s} \wedge \sqrt{1 - s - u} \Big ) \bigg ( 1 - \frac{\Phi_s(P_{\mathrm{MH}})^2 \wedge 1}{8} \bigg )^i.
\end{equation*}
Thus, to complete the proof, we verify that the base case with $c_1 = M_s$ and $c_2 = M_s/s$ holds for all $s \leq u \leq 1 - s$. Let $u \leq s$ and let $A$ be a measurable set with $p(A) = u$ so that $h_0(u) = \mu_0(A) - p(A) \leq M_s$. Similarly, for $u \geq 1 - s$, taking complements,
\begin{equation*}
    h_0(u) = \sup_{p(A) = u} \big ( \mu_0(A) - p(A) \big ) = \sup_{p(B) = 1 - u} \big ( p(B) - \mu_0(B) \big ) \leq M_s.
\end{equation*}
For $s < u < 1 - s$, we use the concavity of $h_0$ as well as the fact that $h_0(0) \geq 0$ and $h_0(1) = 0$, to obtain,
\begin{equation*}
    h_0(u) \leq (u/s) h_0(s), \qquad h_0(u) \leq ((1 - u)/s) h_0(1 - s).
\end{equation*}
Hence, since $h_0(s), h_0(1 - s) \leq M_s$, we obtain,
\begin{align*}
    h_0(u) &\leq \frac{M_s}{s} \big ( u \wedge (1 - u) \big ) \\
    &= M_s + \frac{M_s}{s} \big ( (u - s) \wedge (1 - s - u) \big ) \\
    &\leq M_s + \frac{M_s}{s} \Big ( \sqrt{u - s} \wedge \sqrt{1 - s - u} \Big ).
\end{align*}
The bound in the statement then follows by noting that, for any measurable $A \subseteq \R^d$, $\mu_i(A) - p(A) \leq h_i(p(A))$ and $p(A) - \mu_i(A) = \mu_i(A^c) - p(A^c) \leq h_i(1 - p(A))$, and that $h_i(u) \leq h_0(u) \leq M_s$ for $u \notin [s, 1 - s]$ by (iii).
\end{proof}

We now recall a lower bound on the $s$-conductance of $P_{\mathrm{MH}}$ due to \citet{Wu2022-fs}, for the case where $q$ is the Langevin proposal used in Algorithm~\ref{algorithm:mala_ddm} and $p$ is strongly log-concave. Their result is stated in terms of the isoperimetric constant of $p$, which, for an $m$-strongly log-concave density, is at least $\log(2) \sqrt{m}$ \citep[see][proof of Theorem~1]{Wu2022-fs}.

\begin{lemma}\label{lem:wsc_conductance}
Suppose that $\log p$ is twice continuously differentiable with $m \mathbf{I} \preceq -\nabla^2 \log p \preceq L \mathbf{I}$, and let $P_{\mathrm{MH}}$ be the Metropolis-adjusted kernel with proposal $q(\cdot \mid \mathbf{x}) = \mathcal{N}(\mathbf{x} + \frac{h}{2} \nabla \log p(\mathbf{x}), h \mathbf{I})$. Then, there exist universal constants $c_1, c_2 > 0$ such that, for every $s \in (0, 1/2]$ and
\begin{equation*}
    h \leq \frac{c_1}{L \sqrt{d} \log^2(dL/(ms))},
\end{equation*}
we have $\Phi_s(P_{\mathrm{MH}}) \geq c_2 \sqrt{h m}$.
\end{lemma}
Explicit values of $c_1, c_2$ can be obtained from conditions (33), (34) and the bound (40) in the proof of Theorem~3 of \citet{Wu2022-fs}. Note that their step size is half of ours, and that their bound concerns the lazy version of $P_{\mathrm{MH}}$, whose $s$-conductance is half that of $P_{\mathrm{MH}}$.

\begin{assumption}
\label{ass:data}
The function $\log p_{\text{data}}$ is twice continuously differentiable and there exist $0 < m \leq L$ such that $m \mathbf{I} \preceq -\nabla^2 \log p_{\text{data}}(\mathbf{x}) \preceq L \mathbf{I}$ for every $\mathbf{x} \in \R^d$.
\end{assumption}
Recall from Section~\ref{sec:diffusion} that $p_t$ is the law of $r_t(\mathbf{X}_0 + \sigma_t \mathbf{Z})$, where $\mathbf{X}_0 \sim p_{\text{data}}$ and $\mathbf{Z} \sim \mathcal{N}(\mathbf{0}, \mathbf{I})$ are independent. Under Assumption~\ref{ass:data}, $p_t$ satisfies the same assumption with constants \citep[see][proof of Proposition~7 and Lemma~9]{Gao2025-di},
\begin{equation*}
    m_t = \frac{m}{r_t^2 (1 + m \sigma_t^2)}, \qquad L_t = \frac{L \wedge \sigma_t^{-2}}{r_t^2}.
\end{equation*}
For a measure $\mu$, we write the total variation distance as $\|\mu - p\|_{\mathrm{TV}} = \sup_A |\mu(A) - p(A)|$, where the supremum is over measurable sets $A \subseteq \R^d$. We also write the chi-squared divergence as
\begin{equation}\label{eq:chi}
    \chi^2(\mu \,\|\, p) = \begin{cases}
        \int \big ( \frac{\rmd \mu}{\rmd p}(\mathbf{x}) - 1 \big )^2 p(\mathbf{x}) \rmd \mathbf{x}, & \text{ if } \mu \ll p,\\
        \infty, & \text{ otherwise.}
    \end{cases} 
\end{equation}

Combining Lemmas~\ref{lem:barker_conductance} and \ref{lem:wsc_conductance}, we obtain the following convergence guarantee for the MADM corrector at a fixed noise level.

\begin{proposition}
\label{prop:madm_warm_start}
Fix $t \in [0, 1]$ and let $(\mathbf{X}^{(i)})_{i \geq 0}$ be the iterates of Algorithm~\ref{algorithm:mala_ddm} with the exact score $\nabla \log p_t$ and where the accept-reject step is performed by Algorithm~\ref{algorithm:exact_two_coin}. Suppose that Assumption~\ref{ass:data} holds and $\chi^2(\mu_0 \,\|\, p_t) \leq D$ where we let $\mu_i$ denote the law of $\mathbf{X}^{(i)}$. Given $\varepsilon > 0$ with $\varepsilon^2 \leq 2 D$, suppose that
\begin{equation*}
    h \leq \frac{c_1}{L_t \sqrt{d} \log^2(4 d L_t D/(m_t \varepsilon^2))},
\end{equation*}
where $c_1$ is the constant of Lemma~\ref{lem:wsc_conductance}. Then, there exists a universal constant $c_3 > 0$ such that $\|\mu_i - p_t\|_{\mathrm{TV}} \leq \varepsilon$ whenever,
\begin{equation*}
    i \geq \frac{1}{c_3 h m_t} \log \bigg ( \frac{4 D}{\varepsilon^2} \bigg ).
\end{equation*}
\end{proposition}
\begin{proof}
By Theorem~\ref{thm:two_coin_barker}, each accept-reject step accepts with probability $\alpha_{\mathrm{B}}$. Hence, $(\mathbf{X}^{(i)})_{i \geq 0}$ is a Markov chain with transition kernel $P_{\mathrm{B}}$, with target $p = p_t$ and proposal $q(\cdot \mid \mathbf{x}) = \mathcal{N}(\mathbf{x} + \frac{h}{2} \nabla \log p_t(\mathbf{x}), h \mathbf{I})$. Thus, we may apply Lemma~\ref{lem:barker_conductance}. To bound $M_s$, we note that for any measurable $A$, the Cauchy--Schwarz inequality gives,
\begin{align*}
    |\mu_0(A) - p_t(A)| &= \bigg | \int_A \bigg ( \frac{\rmd \mu_0}{\rmd p_t}(\mathbf{x}) - 1 \bigg ) p_t(\mathbf{x}) \rmd \mathbf{x} \bigg |\\
    &\leq \sqrt{D \, p_t(A)}.
\end{align*}
Thus, setting $s = \varepsilon^2/(4 D)$, which lies in $(0, 1/2]$ since $\varepsilon^2 \leq 2 D$, we obtain $M_s \leq \sqrt{s D} = \varepsilon/2$.

Since $p_t$ satisfies Assumption~\ref{ass:data} with constants $m_t$ and $L_t$, Lemma~\ref{lem:wsc_conductance} gives $\Phi_s(P_{\mathrm{MH}}) \geq c_2 \sqrt{h m_t}$ and hence there exists a universal constant $c_3 > 0$ such that,
\begin{equation*}
    \frac{\Phi_s(P_{\mathrm{MH}})^2 \wedge 1}{8} \geq \frac{c_2^2 h m_t \wedge 1}{8} \geq c_3 h m_t.
\end{equation*}
Applying Lemma~\ref{lem:barker_conductance} with $p = p_t$, we obtain that, for every measurable $A$,
\begin{equation*}
    |\mu_i(A) - p_t(A)| \leq \frac{\varepsilon}{2} + \frac{2 D}{\varepsilon} \big ( 1 - c_3 h m_t \big )^i \leq \frac{\varepsilon}{2} + \frac{2 D}{\varepsilon} e^{-c_3 h m_t i},
\end{equation*}
which is at most $\varepsilon$ for $i$ as in the statement.
\end{proof}

\subsection{Oracle complexity of the MADM-based sampler}\label{appendix:oracle_complexity}

We now analyse the complexity of the MADM-based predictor-corrector sampler in setting of Assumption \ref{ass:data}. We work in the VP setting of Section~\ref{sec:diffusion} and fix noise levels $1 \geq t_0 > t_1 > \dots > t_N \geq 0$. For $k \in \{1, \dots, N\}$, let $\beta_k = 1 - r_{t_{k-1}}^2 / r_{t_k}^2$, so that if $\mathbf{X} \sim p_{t_k}$ and $\mathbf{Z} \sim \mathcal{N}(\mathbf{0}, \mathbf{I})$ are independent, then $\sqrt{1 - \beta_k} \mathbf{X} + \sqrt{\beta_k} \mathbf{Z} \sim p_{t_{k-1}}$. For each noise level $t_k$, we perform a predictor step, followed by $n_k$ steps of the MADM corrector. The predictor is the ancestral sampler of \citet{ho2020ddpm} with the exact score, whose step from noise level $t_{k-1}$ to $t_k$ has transition kernel,
\begin{equation*}
    Q_k(\mathbf{x}, \cdot) = \mathcal{N} \bigg ( \frac{\mathbf{x} + \beta_k \nabla \log p_{t_{k-1}}(\mathbf{x})}{\sqrt{1 - \beta_k}}, \beta_k \mathbf{I} \bigg ).
\end{equation*}
The corrector at level $k$ is Algorithm~\ref{algorithm:mala_ddm}, where the accept-reject step is performed by Algorithm~\ref{algorithm:exact_two_coin}. By Theorem~\ref{thm:two_coin_barker}, its transition kernel is the Barker-adjusted kernel $P_k$ with target $p_{t_k}$ and proposal $q_k(\cdot \mid \mathbf{x}) = \mathcal{N}(\mathbf{x} + \frac{h_k}{2} \nabla \log p_{t_k}(\mathbf{x}), h_k \mathbf{I})$. Since the sampler is initialised from $\mathcal{N}(\mathbf{0}, \mathbf{I})$, the law of its output at level $k$ is given by,
\begin{equation*}
    \nu_k = \nu_{k-1} Q_k P_k^{n_k}, \qquad \nu_0 = \mathcal{N}(\mathbf{0}, \mathbf{I}).
\end{equation*}

The following lemma decomposes the error of the sampler into the error of the initialisation and the error made at each level when the sampler is started from the exact marginal. In particular, the warm start needed at level $k$ concerns $p_{t_{k-1}} Q_k$ rather than the law $\nu_{k-1} Q_k$ of the sampler.

\begin{lemma}
\label{lem:decomposition}
For every $N \geq 1$,
\begin{equation*}
    \|\nu_N - p_{t_N}\|_{\mathrm{TV}} \leq \|\nu_0 - p_{t_0}\|_{\mathrm{TV}} + \sum_{k = 1}^N \big \| p_{t_{k-1}} Q_k P_k^{n_k} - p_{t_k} \big \|_{\mathrm{TV}}.
\end{equation*}
\end{lemma}
\begin{proof}
Fix $k \in \{1, \dots, N\}$. By the triangle inequality,
\begin{equation*}
    \|\nu_k - p_{t_k}\|_{\mathrm{TV}} \leq \big \| \nu_{k-1} Q_k P_k^{n_k} - p_{t_{k-1}} Q_k P_k^{n_k} \big \|_{\mathrm{TV}} + \big \| p_{t_{k-1}} Q_k P_k^{n_k} - p_{t_k} \big \|_{\mathrm{TV}}.
\end{equation*}
For any transition kernel $P$, probability measures $\mu, \nu$ and measurable $A$, since $P(\cdot, A)$ takes values in $[0, 1]$,
\begin{equation*}
    |\mu P(A) - \nu P(A)| = \bigg | \int P(\mathbf{x}, A) \rmd (\mu - \nu)(\mathbf{x}) \bigg | \leq \|\mu - \nu\|_{\mathrm{TV}}.
\end{equation*}
Hence, the first term is at most $\|\nu_{k-1} - p_{t_{k-1}}\|_{\mathrm{TV}}$ and the result follows by induction on $k$.
\end{proof}

Next, we bound the warm start at each level. The key observation is that, by Tweedie's identity, the predictor step $Q_k$ has the same mean as the reverse transition of the forward process, so that it only errs in the shape of this transition.

\begin{lemma}
\label{lem:warm_start_levels}
Suppose that Assumption~\ref{ass:data} holds, let $k \in \{1, \dots, N\}$ and let $\delta_k = |1 - m_{t_k}| \vee |L_{t_k} - 1|$. If $\beta_k \delta_k \leq 1/4$, then
\begin{equation*}
    \log \big ( 1 + \chi^2(p_{t_{k-1}} Q_k \,\|\, p_{t_k}) \big ) \leq 7 d \beta_k^2 \delta_k^2.
\end{equation*}
\end{lemma}
\begin{proof}
For brevity, we write $\beta = \beta_k$ and $\delta = \delta_k$. We begin by noting that we can describe both $p_{t_{k-1}} Q_k$ and $p_{t_k}$ as mixture distributions over $\mathbf{y} \sim p_{t_{k-1}}$. Indeed, for any measurable $A \subset \R^d$,
\begin{equation*}
    p_{t_{k-1}} Q_k(A) = \int Q_k(\mathbf{y}, A) p_{t_{k-1}}(\rmd \mathbf{y}), \qquad p_{t_k}(A) = \int p_{t_k|t_{k-1}}(A|\mathbf{y}) p_{t_{k-1}}(\rmd \mathbf{y}).
\end{equation*}
Thus, by the joint convexity of the chi-squared divergence, we obtain,
\begin{equation*}
    \chi^2(p_{t_{k-1}} Q_k \,\|\, p_{t_k}) \leq \int \chi^2(Q_k(\mathbf{y}, \cdot) \,\|\, p_{t_k|t_{k-1}}(\cdot|\mathbf{y}) ) p_{t_{k-1}}(\rmd \mathbf{y}).
\end{equation*}
Therefore, it suffices to bound $\log(1 + \chi^2(Q_k(\mathbf{y}, \cdot) \,\|\, p_{t_k|t_{k-1}}(\cdot|\mathbf{y}) ))$ uniformly in $\mathbf{y} \in \R^d$.

Fix $\mathbf{y} \in \R^d$. By Tweedie's identity, we have that the mean of $p_{t_k|t_{k-1}}(\cdot|\mathbf{y})$ is,
\begin{equation*}
    \widehat{\mathbf{x}} = \frac{\mathbf{y} + \beta \nabla \log p_{t_{k-1}}(\mathbf{y})}{\sqrt{1 - \beta}},
\end{equation*}
which coincides with the mean of $Q_k(\mathbf{y}, \cdot) = \mathcal{N}(\widehat{\mathbf{x}}, \beta \mathbf{I})$. Moreover, by Bayes' rule, $p_{t_k|t_{k-1}}(\cdot|\mathbf{y})$ has density proportional to $p_{t_k}(\mathbf{x}) \exp(-\|\mathbf{y} - \sqrt{1 - \beta} \mathbf{x}\|^2/(2 \beta))$, and so its density with respect to $Q_k(\mathbf{y}, \cdot)$ satisfies,
\begin{equation*}
    \frac{\rmd p_{t_k|t_{k-1}}(\cdot|\mathbf{y})}{\rmd Q_k(\mathbf{y}, \cdot)} = \frac{e^{-V}}{\E_{\mathbf{x} \sim Q_k(\mathbf{y}, \cdot)} \big [ e^{-V(\mathbf{x})} \big ]}, \qquad V(\mathbf{x}) = -\log p_{t_k}(\mathbf{x}) + \frac{1}{2 \beta} \big \| \mathbf{y} - \sqrt{1 - \beta} \mathbf{x} \big \|^2 - \frac{1}{2 \beta} \|\mathbf{x} - \widehat{\mathbf{x}}\|^2.
\end{equation*}
Thus, we upper bound the chi-squared divergence by,
\begin{align*}
    1 + \chi^2 \big ( Q_k(\mathbf{y}, \cdot) \,\|\, p_{t_k|t_{k-1}}(\cdot|\mathbf{y}) \big ) &= \E_{\mathbf{x} \sim Q_k(\mathbf{y}, \cdot)} \bigg [ \frac{\rmd Q_k(\mathbf{y}, \cdot)}{\rmd p_{t_k|t_{k-1}}(\cdot|\mathbf{y})}(\mathbf{x}) \bigg ] \\
    &= \E_{\mathbf{x} \sim Q_k(\mathbf{y}, \cdot)} \big [ e^{V(\mathbf{x})} \big ] \E_{\mathbf{x} \sim Q_k(\mathbf{y}, \cdot)} \big [ e^{-V(\mathbf{x})} \big ] \\
    &= \E \big [ e^{V(\mathbf{x}) - V(\mathbf{x}')} \big ],
\end{align*}
where $\mathbf{x}, \mathbf{x}'$ are independent samples of $Q_k(\mathbf{y}, \cdot)$. Using the Gaussian inequality of Maurey and Pisier \citep[e.g. see][]{Bobkov2024-cj}, we thus obtain,
\begin{align*}
    1 + \chi^2 \big ( Q_k(\mathbf{y}, \cdot) \,\|\, p_{t_k|t_{k-1}}(\cdot|\mathbf{y}) \big ) &\leq \E \bigg [ \exp \bigg ( \frac{\pi}{2} \langle \nabla V(\mathbf{x}), \mathbf{x}' - \widehat{\mathbf{x}} \rangle \bigg ) \bigg ]\\
    &\leq \E \bigg [ \exp \bigg ( \frac{\pi^2 \beta}{8} \|\nabla V(\mathbf{x})\|^2 \bigg ) \bigg ].
\end{align*}
Since $\nabla^2 V = -\nabla^2 \log p_{t_k} - \mathbf{I}$ and $p_{t_k}$ satisfies Assumption~\ref{ass:data} with constants $m_{t_k}$ and $L_{t_k}$, we have $-\delta \mathbf{I} \preceq \nabla^2 V \preceq \delta \mathbf{I}$ and thus, $\|\nabla V(\mathbf{x})\|^2 \leq 2 \|\nabla V(\widehat{\mathbf{x}})\|^2 + 2 \delta^2 \|\mathbf{x} - \widehat{\mathbf{x}}\|^2$. Thus, we obtain the upper bound,
\begin{align*}
    \log \Big ( 1 + \chi^2 \big ( Q_k(\mathbf{y}, \cdot) \,\|\, p_{t_k|t_{k-1}}(\cdot|\mathbf{y}) \big ) \Big ) &\leq \frac{\pi^2 \beta}{4} \|\nabla V(\widehat{\mathbf{x}})\|^2 + \log \E \bigg [ \exp \bigg ( \frac{\pi^2 \beta \delta^2}{4} \| \mathbf{x} - \widehat{\mathbf{x}} \|^2 \bigg ) \bigg ] \\
    &\leq \frac{\pi^2 \beta}{4} \|\nabla V(\widehat{\mathbf{x}})\|^2 - \frac{d}{2} \log \bigg ( 1 - \frac{\pi^2 \beta^2 \delta^2}{2} \bigg ),
\end{align*}
where in the last line we use that $\mathbf{x} - \widehat{\mathbf{x}} \sim \mathcal{N}(\mathbf{0}, \beta \mathbf{I})$ and $\pi^2 \beta^2 \delta^2 / 2 < 1$.

It remains to bound $\|\nabla V(\widehat{\mathbf{x}})\|$, for which we use that the means of $Q_k(\mathbf{y}, \cdot)$ and $p_{t_k|t_{k-1}}(\cdot|\mathbf{y})$ coincide. By integration by parts, the score of $p_{t_k|t_{k-1}}(\cdot|\mathbf{y})$ has mean zero under it, that is,
\begin{equation*}
    \E_{\mathbf{x} \sim p_{t_k|t_{k-1}}(\cdot|\mathbf{y})} \big [ \nabla V(\mathbf{x}) + (\mathbf{x} - \widehat{\mathbf{x}})/\beta \big ] = \mathbf{0},
\end{equation*}
and since the mean of $p_{t_k|t_{k-1}}(\cdot|\mathbf{y})$ is $\widehat{\mathbf{x}}$, it follows that $\E_{\mathbf{x} \sim p_{t_k|t_{k-1}}(\cdot|\mathbf{y})} [\nabla V(\mathbf{x})] = \mathbf{0}$. Hence,
\begin{equation*}
    \|\nabla V(\widehat{\mathbf{x}})\| = \big \| \E_{\mathbf{x} \sim p_{t_k|t_{k-1}}(\cdot|\mathbf{y})} [ \nabla V(\widehat{\mathbf{x}}) - \nabla V(\mathbf{x}) ] \big \| \leq \delta \, \E_{\mathbf{x} \sim p_{t_k|t_{k-1}}(\cdot|\mathbf{y})} \|\mathbf{x} - \widehat{\mathbf{x}}\|.
\end{equation*}
The negative log-density of $p_{t_k|t_{k-1}}(\cdot|\mathbf{y})$ has Hessian $-\nabla^2 \log p_{t_k} + ((1 - \beta)/\beta) \mathbf{I} \succeq ((1 - \beta \delta)/\beta) \mathbf{I}$, and so, by the Brascamp--Lieb inequality \citep[e.g. see]{Bakry2014-ut},
\begin{equation*}
    \E_{\mathbf{x} \sim p_{t_k|t_{k-1}}(\cdot|\mathbf{y})} \|\mathbf{x} - \widehat{\mathbf{x}}\|^2 \leq d \beta/(1 - \beta \delta).
\end{equation*}
Thus, $\|\nabla V(\widehat{\mathbf{x}})\|^2 \leq d \beta \delta^2/(1 - \beta \delta)$. Combining these bounds and using that $\beta \delta \leq 1/4$ and $-\log(1 - u) \leq u/(1 - u)$ for $u \in [0, 1)$, we conclude that,
\begin{align*}
    \log \Big ( 1 + \chi^2 \big ( Q_k(\mathbf{y}, \cdot) \,\|\, p_{t_k|t_{k-1}}(\cdot|\mathbf{y}) \big ) \Big ) &\leq \frac{\pi^2 d \beta^2 \delta^2}{4 (1 - \beta \delta)} + \frac{\pi^2 d \beta^2 \delta^2}{4 (1 - \pi^2 \beta^2 \delta^2/2)} \\
    &\leq 7 d \beta^2 \delta^2.
\end{align*}
Since this bound is uniform in $\mathbf{y}$, the result follows.
\end{proof}

Combining Lemmas~\ref{lem:decomposition} and \ref{lem:warm_start_levels} with Proposition~\ref{prop:madm_warm_start}, we obtain the following convergence guarantee for the full sampler. For simplicity, we assume that $m \leq 1 \leq L$ and write $\kappa = L/m$.

\begin{proposition}
\label{prop:end_to_end}
Suppose that Assumption~\ref{ass:data} holds with $m \leq 1 \leq L$ and let $\varepsilon \in (0, 1)$. Let $t_0$ be such that $r_{t_0}^2 = \varepsilon^2/(2(\E\|\mathbf{X}_0\|^2 + d))$, let $N = \lceil 4 L \sqrt{d} \log(1/r_{t_0}^2) \rceil$ and choose the remaining noise levels such that $r_{t_k}^2 = r_{t_0}^{2(1 - k/N)}$. Suppose also that the step size $h_k = h$ and the number of corrector steps $n_k = n$ do not depend on $k$ and satisfy,
\begin{equation*}
    h \leq \frac{c_1}{2 L \sqrt{d} \log^2(896 \kappa d^2 L^2 (\E\|\mathbf{X}_0\|^2 + d)/\varepsilon^4)}, \qquad n \geq \frac{1}{c_3 h m} \log \bigg ( \frac{448 d L^2 (\E\|\mathbf{X}_0\|^2 + d)}{\varepsilon^4} \bigg ),
\end{equation*}
where $c_1$ and $c_3$ are the constants of Lemma~\ref{lem:wsc_conductance} and Proposition~\ref{prop:madm_warm_start}. Then, $\|\nu_N - p_{\text{data}}\|_{\mathrm{TV}} \leq \varepsilon$.
\end{proposition}
Here, we assume that $r_1^2 \leq \varepsilon^2/(2(\E\|\mathbf{X}_0\|^2 + d))$, so that $t_0$ exists. Taking the largest admissible step size, the total number of iterations $N(n + 1)$ is $O(L \kappa d)$, up to logarithmic factors.

\begin{proof}
By Lemma~\ref{lem:decomposition}, we have that,
\begin{equation*}
    \|\nu_N - p_{\text{data}}\|_{\mathrm{TV}} \leq \|\nu_0 - p_{t_0}\|_{\mathrm{TV}} + \sum_{k = 1}^N \big \| p_{t_{k-1}} Q_k P_k^n - p_{t_k} \big \|_{\mathrm{TV}},
\end{equation*}
so it suffices to show that the first term is at most $\varepsilon/2$ and that each term in the sum is at most $\varepsilon/(2N)$.

For the first term, we write $r = r_{t_0}$, so that $p_{t_0}$ is the law of $r \mathbf{X}_0 + \sqrt{1 - r^2} \mathbf{Z}$. Since the KL divergence is convex in its first argument, comparing each Gaussian $\mathcal{N}(r \mathbf{x}, (1 - r^2) \mathbf{I})$ with $\nu_0 = \mathcal{N}(\mathbf{0}, \mathbf{I})$ gives,
\begin{align*}
    \mathrm{KL}(p_{t_0} \,\|\, \nu_0) &\leq \frac{1}{2} \E \big [ r^2 \|\mathbf{X}_0\|^2 - d r^2 - d \log(1 - r^2) \big ] \\
    &\leq \frac{r^2}{2} \big ( \E\|\mathbf{X}_0\|^2 + d \big ),
\end{align*}
where in the last line we use that $-\log(1 - x) \leq x/(1 - x)$ and $r^2 \leq 1/2$. Hence, by Pinsker's inequality, $\|\nu_0 - p_{t_0}\|_{\mathrm{TV}} \leq \sqrt{\mathrm{KL}(p_{t_0} \,\|\, \nu_0)/2} \leq \varepsilon/2$.

For the terms in the sum, fix $k \in \{1, \dots, N\}$. By the choice of noise levels, we have $\beta_k = \beta := 1 - r_{t_0}^{2/N}$, so that $(1 - \beta)^N = r_{t_0}^2$. Since $m \leq 1 \leq L$ and $r_t^2 (1 + \sigma_t^2) = 1$, the expressions for $m_t$ and $L_t$ give,
\begin{equation*}
    m \leq m_{t_k} \leq 1 \leq L_{t_k} \leq 1 + L,
\end{equation*}
so that the quantity $\delta_k$ of Lemma~\ref{lem:warm_start_levels} satisfies $\delta_k \leq L$. Hence, since $\beta \leq -\log(1 - \beta) = \log(1/r_{t_0}^2)/N \leq 1/(4 L \sqrt{d})$ by the choice of $N$, we have $\beta \delta_k \leq 1/(4 \sqrt{d}) \leq 1/4$. Thus, from Lemma~\ref{lem:warm_start_levels} it follows that,
\begin{equation*}
    \chi^2(p_{t_{k-1}} Q_k \,\|\, p_{t_k}) \leq e^{7 d \beta^2 \delta_k^2} - 1 \leq 14 d L^2 \beta^2 =: D.
\end{equation*}
We now apply Proposition~\ref{prop:madm_warm_start} at noise level $t_k$, with initial law $\mu_0 = p_{t_{k-1}} Q_k$ and accuracy $\varepsilon/(2N)$. This first requires that $\varepsilon^2 \leq 8 N^2 D = 112 d L^2 N^2 \beta^2$, which holds since $\varepsilon < 1$ and $N \beta \geq \log(2)/2$. The remaining conditions of Proposition~\ref{prop:madm_warm_start} read,
\begin{equation*}
    h \leq \frac{c_1}{L_{t_k} \sqrt{d} \log^2(16 d L_{t_k} N^2 D/(m_{t_k} \varepsilon^2))}, \qquad n \geq \frac{1}{c_3 h m_{t_k}} \log \bigg ( \frac{16 N^2 D}{\varepsilon^2} \bigg ).
\end{equation*}
Since $N \beta \leq \log(1/r_{t_0}^2) \leq 1/r_{t_0}$, we have $16 N^2 D = 224 d L^2 N^2 \beta^2 \leq 448 d L^2 (\E\|\mathbf{X}_0\|^2 + d)/\varepsilon^2$. Thus, since also $L_{t_k} \leq 2L$, $L_{t_k}/m_{t_k} \leq 2 \kappa$ and $m_{t_k} \geq m$, these conditions are implied by those in the statement. Hence,
\begin{equation*}
    \big \| p_{t_{k-1}} Q_k P_k^n - p_{t_k} \big \|_{\mathrm{TV}} \leq \frac{\varepsilon}{2N},
\end{equation*}
completing the proof.
\end{proof}

We now convert the number of iterations in Proposition~\ref{prop:end_to_end} into a bound on the number of score queries. By Proposition~\ref{prop:two_coin_cost}, the number of score queries made by Algorithm~\ref{algorithm:exact_two_coin} is governed by the bound $C$, so we first give a choice of $C$ under Assumption~\ref{ass:data} and bound its size at the proposal of Algorithm~\ref{algorithm:mala_ddm}.

\begin{lemma}
\label{lem:lipschitz_C}
Suppose that Assumption~\ref{ass:data} holds and fix $t \in [0, 1]$. Then, Assumption~\ref{ass:score_bound} holds for the score $\nabla \log p_t$ with,
\begin{equation*}
    C(\mathbf{x}, \widetilde{\mathbf{x}}) = \frac{L_t}{2} \|\widetilde{\mathbf{x}} - \mathbf{x}\|^2.
\end{equation*}
Furthermore, if $\widetilde{\mathbf{x}} = \mathbf{x} + \frac{h}{2} \nabla \log p_t(\mathbf{x}) + \sqrt{h} \mathbf{z}$ for some $\mathbf{z} \in \R^d$, then,
\begin{equation*}
    C(\mathbf{x}, \widetilde{\mathbf{x}}) \leq \frac{L_t h^2}{4} \|\nabla \log p_t(\mathbf{x})\|^2 + L_t h \|\mathbf{z}\|^2.
\end{equation*}
\end{lemma}
\begin{proof}
For $u \in [0, 1]$, let $f(u) = \langle \nabla \log p_t(\mathbf{x} + u(\widetilde{\mathbf{x}} - \mathbf{x})), \widetilde{\mathbf{x}} - \mathbf{x} \rangle$, so that $\hat{I}_{\Delta}^{(2)} = \frac{1}{2}(f(0) + f(1))$. Since $p_t$ satisfies Assumption~\ref{ass:data} with constants $m_t$ and $L_t$, we have $|f'(u)| \leq L_t \|\widetilde{\mathbf{x}} - \mathbf{x}\|^2$ and thus, $|f(u) - f(0)| \leq L_t \|\widetilde{\mathbf{x}} - \mathbf{x}\|^2 u$ and $|f(u) - f(1)| \leq L_t \|\widetilde{\mathbf{x}} - \mathbf{x}\|^2 (1 - u)$. Hence,
\begin{equation*}
    \big | f(u) - \hat{I}_{\Delta}^{(2)} \big | \leq \frac{1}{2} |f(u) - f(0)| + \frac{1}{2} |f(u) - f(1)| \leq \frac{L_t}{2} \|\widetilde{\mathbf{x}} - \mathbf{x}\|^2,
\end{equation*}
which establishes the first claim. The second claim then follows from the first by using that $\|\mathbf{a} + \mathbf{b}\|^2 \leq 2 \|\mathbf{a}\|^2 + 2 \|\mathbf{b}\|^2$.
\end{proof}

\begin{proposition}
\label{prop:oracle_complexity}
Suppose that the conditions of Proposition~\ref{prop:end_to_end} hold, that at each level $k$ the accept-reject step of the corrector uses the bound $C$ of Lemma~\ref{lem:lipschitz_C} with $t = t_k$, and let $\delta \in (0, 1)$. Suppose further that,
\begin{equation*}
    h \leq \frac{1}{32 L (d + \ell + \sqrt{\kappa \ell})}, \qquad \ell = \log(8 N n/\delta).
\end{equation*}
Then, with probability at least $1 - \delta/2 - N \varepsilon$, the total number of score queries made by the sampler is at most $15 N n \ell$.
\end{proposition}
Since $\ell$ depends on $n$ only logarithmically, taking the largest admissible step size, $h$ is of order $1/(L(d + \sqrt{\kappa}))$ and the total number of score queries is $O(L \kappa \sqrt{d} (d + \sqrt{\kappa}))$, up to logarithmic factors. In particular, this is $O(L \kappa d^{3/2})$ when $\kappa \leq d^2$. Moreover, since the conditions depend on $\varepsilon$ only logarithmically, taking $\varepsilon \leq \delta/(2N)$ ensures that the bound holds with probability at least $1 - \delta$, at the cost of logarithmic factors.

\begin{proof}
At each level, the predictor evaluates the score once, and each corrector step evaluates the score at the current state and at the proposal. Thus, writing $J_{k, i}$ for the number of score queries made by Algorithm~\ref{algorithm:exact_two_coin} in the $i$-th corrector step at level $k$, the total number of score queries is at most,
\begin{equation*}
    N + 2 N n + \sum_{k = 1}^N \sum_{i = 1}^n J_{k, i}.
\end{equation*}

We begin by bounding $J_{k, i}$ in terms of $C$. Fix $k$ and $i$, let $\mathbf{x}$ be the state of the chain before the $i$-th corrector step at level $k$ and let $\widetilde{\mathbf{x}}$ be the proposal. By Lemma~\ref{lem:lipschitz_C}, $C = C(\mathbf{x}, \widetilde{\mathbf{x}})$ satisfies Assumption~\ref{ass:score_bound}, so that Proposition~\ref{prop:two_coin_cost} applies. Writing $H$ and $r$ as in Section~\ref{sec:two_coin}, we have $1 + H e^{\hat{I}_{\Delta}^{(2)} + C} \leq e^{2C}(1 + Hr)$ since $|\log r - \hat{I}_{\Delta}^{(2)}| \leq C$. Hence, conditionally on $\mathbf{x}$ and $\widetilde{\mathbf{x}}$, with probability at least $1 - \delta/(8Nn)$,
\begin{equation*}
    J_{k, i} \leq 4 C e^{2C} \log(16Nn/\delta) + 4C + \log(16Nn/\delta).
\end{equation*}
In particular, if $C \leq 1/2$, then $J_{k, i} \leq (2e + 1) \log(16Nn/\delta) + 2 \leq 12 \ell$, where we use that $\ell \geq \log 4$.

Thus, since $N + 2 N n + 12 N n \ell \leq 15 N n \ell$ by $\ell \geq \log 4$, it suffices to show that, with probability at least $1 - \delta/2 - N \varepsilon$, we have $C \leq 1/2$ and the above bound on $J_{k, i}$ holds at every corrector step.

To this end, we keep $k$ and $i$ fixed and let $\mathbf{z} \sim \mathcal{N}(\mathbf{0}, \mathbf{I})$ be the Gaussian increment of the proposal, so that $\widetilde{\mathbf{x}} = \mathbf{x} + \frac{h}{2} \nabla \log p_{t_k}(\mathbf{x}) + \sqrt{h} \mathbf{z}$. Let $T = \sqrt{2 d L} + 4 L \sqrt{\ell/m}$ and $S_k = \{\mathbf{y} \in \R^d : \|\nabla \log p_{t_k}(\mathbf{y})\| > T\}$, and consider the events,
\begin{equation*}
    A_{k, i} = \{ \mathbf{x} \in S_k \}, \qquad Z_{k, i} = \big \{ \|\mathbf{z}\| > \sqrt{d} + \sqrt{2 \ell} \big \},
\end{equation*}
and let $B_{k, i}$ be the event that the above bound on $J_{k, i}$ fails. We also let $E_k = \bigcup_{i = 1}^n (A_{k, i} \cup Z_{k, i} \cup B_{k, i})$ and, for a probability measure $\mu$, we write $\prob_\mu$ for the probability of an event when level $k$ is started from $\mu$, that is, when the predictor at level $k$ is applied to a sample from $\mu$.

Outside of $A_{k, i} \cup Z_{k, i}$, Lemma~\ref{lem:lipschitz_C} and $L_{t_k} \leq 2L$ give,
\begin{align*}
    C &\leq \frac{L h^2}{2} \|\nabla \log p_{t_k}(\mathbf{x})\|^2 + 2 L h \|\mathbf{z}\|^2 \\
    &\leq \frac{L h^2 T^2}{2} + 4 L h (d + 2 \ell) \\
    &\leq \frac{1}{8} + \frac{1}{4} \leq \frac{1}{2},
\end{align*}
where in the last line we use that, by the condition on $h$, $h T \leq 1/(2 \sqrt{L})$ and $4 L h (d + 2 \ell) \leq 1/4$. Hence, outside of $E_k$, we have $C \leq 1/2$ and the above bound on $J_{k, i}$ holds at every corrector step at level $k$.

We now bound the probability of $E_k$ when level $k$ is started from $p_{t_{k-1}}$. Since $\mathbf{z} \mapsto \|\mathbf{z}\|$ is $1$-Lipschitz and $\E \|\mathbf{z}\| \leq \sqrt{d}$, Gaussian concentration \citep[see][]{Bakry2014-ut} gives $\prob_{p_{t_{k-1}}}(Z_{k, i}) \leq e^{-\ell} = \delta/(8Nn)$, and by the bound on $J_{k, i}$ above, $\prob_{p_{t_{k-1}}}(B_{k, i}) \leq \delta/(8Nn)$. For $A_{k, i}$, we first bound $p_{t_k}(S_k)$. Since $p_{t_k}$ is $m_{t_k}$-strongly log-concave and $\mathbf{y} \mapsto \|\nabla \log p_{t_k}(\mathbf{y})\|$ is $L_{t_k}$-Lipschitz, Gaussian concentration gives, for $\mathbf{Y} \sim p_{t_k}$ and $\rho = 4 L \sqrt{\ell/m}$,
\begin{align*}
    p_{t_k}(S_k) &\leq \prob \big ( \|\nabla \log p_{t_k}(\mathbf{Y})\| \geq \E \|\nabla \log p_{t_k}(\mathbf{Y})\| + \rho \big ) \\
    &\leq \exp \bigg ( -\frac{m_{t_k} \rho^2}{2 L_{t_k}^2} \bigg ) \\
    &\leq \bigg ( \frac{\delta}{8Nn} \bigg )^2,
\end{align*}
where in the first line we use that, by integration by parts, $\E \|\nabla \log p_{t_k}(\mathbf{Y})\|^2 = -\E [\Delta \log p_{t_k}(\mathbf{Y})] \leq d L_{t_k} \leq 2 d L$, and in the last line we use that $m_{t_k} \geq m$ and $L_{t_k} \leq 2L$. To transfer this bound to the chain, recall from the proof of Proposition~\ref{prop:end_to_end} that $\chi^2(p_{t_{k-1}} Q_k \,\|\, p_{t_k}) \leq D = 14 d L^2 \beta^2$, where $\beta \leq 1/(4 L \sqrt{d})$ so that $D \leq 1$. Let $\mu_{k, i} = p_{t_{k-1}} Q_k P_k^{i-1}$ denote the law of $\mathbf{x}$ when level $k$ is started from $p_{t_{k-1}}$. Since $p_{t_k} P_k = p_{t_k}$, the data processing inequality gives $\chi^2(\mu_{k, i} \,\|\, p_{t_k}) \leq D \leq 1$, and so, using the Cauchy--Schwarz inequality as in the proof of Proposition~\ref{prop:madm_warm_start},
\begin{equation*}
    \prob_{p_{t_{k-1}}}(A_{k, i}) = \mu_{k, i}(S_k) \leq p_{t_k}(S_k) + \sqrt{p_{t_k}(S_k)} \leq \frac{\delta}{4Nn}.
\end{equation*}
Thus, by a union bound over the corrector steps at level $k$,
\begin{equation*}
    \prob_{p_{t_{k-1}}}(E_k) \leq \sum_{i = 1}^n \bigg ( \frac{\delta}{4Nn} + \frac{\delta}{8Nn} + \frac{\delta}{8Nn} \bigg ) = \frac{\delta}{2N}.
\end{equation*}

Finally, the sampler starts level $k$ from $\nu_{k-1}$ rather than $p_{t_{k-1}}$. Writing $K_k(\mathbf{y}, E_k)$ for the probability of $E_k$ when level $k$ is started from $\mathbf{y}$, which takes values in $[0, 1]$, we obtain as in the proof of Lemma~\ref{lem:decomposition},
\begin{align*}
    \big | \prob_{\nu_{k-1}}(E_k) - \prob_{p_{t_{k-1}}}(E_k) \big | &= \bigg | \int K_k(\mathbf{y}, E_k) \rmd (\nu_{k-1} - p_{t_{k-1}})(\mathbf{y}) \bigg | \\
    &\leq \|\nu_{k-1} - p_{t_{k-1}}\|_{\mathrm{TV}}.
\end{align*}
By Lemma~\ref{lem:decomposition} and the bounds in the proof of Proposition~\ref{prop:end_to_end}, $\|\nu_{k-1} - p_{t_{k-1}}\|_{\mathrm{TV}} \leq \varepsilon/2 + (k - 1)\varepsilon/(2N) \leq \varepsilon$. Hence, by a union bound over $k$, the probability that some $E_k$ occurs is at most $\delta/2 + N \varepsilon$, completing the proof.
\end{proof}

As noted in Section \ref{sec:theory}, the key feature of Proposition \ref{prop:oracle_complexity} is the polylogarithmic dependence on $1/\varepsilon$ compared to results on diffusion samplers built from predictor-only steps which require $\widetilde{\mathcal{O}}(d/\varepsilon^2)$ score queries \citep{Benton2024-np}, or $\widetilde{\mathcal{O}}(L^2 d/\varepsilon^2)$ when the score is $L$-Lipschitz \citep{chen2023sampling}, and the dimension dependence can be reduced to $\widetilde{\mathcal{O}}(L^2 \sqrt{d}/\varepsilon)$ by using the probability flow ODE as the predictor together with an underdamped Langevin corrector \citep{chen2023probability}. This is the same mechanism that separates ULA, whose complexity is polynomial in $1/\varepsilon$ \citep{durmus2017nonasymptotic,mousavi2023towards}, from MALA, whose complexity is polylogarithmic in $1/\varepsilon$ \citep{Wu2022-fs}. Recently, \citet{chen2026sampling} also obtained polylogarithmic guarantees for diffusion models by instead removing the bias of each predictor step with rejection sampling. Their guarantees hold under minimal assumptions on the data and with approximate scores, using $\widetilde{\mathcal{O}}(d_\star)$ score queries, where $d_\star \leq d$ is an intrinsic dimension of the data, which equals $d$ for data with full support, as under Assumption~\ref{ass:data}.

The price of exact correction in Proposition~\ref{prop:main_complexity} is a worse dependence on $d$ and $\kappa$. The factor $d^{3/2}$ is the product of the $\sqrt{d}$ noise levels needed for the predictor to provide a warm start and the factor $d$ in the number of corrector steps $n_k$. The latter comes from the cost of Algorithm~\ref{algorithm:exact_two_coin}, which grows exponentially in $C$ (Proposition~\ref{prop:two_coin_cost}). At a typical proposal, $C = \frac{L_t}{2} \|\widetilde{\mathbf{x}} - \mathbf{x}\|^2$ is of order $L h d$, so we require $h \lesssim 1/(L d)$, whereas MALA, which evaluates the density, mixes with step sizes of order $1/(L \sqrt{d})$ in $\widetilde{\mathcal{O}}(\kappa \sqrt{d})$ steps \citep{Wu2022-fs}. In the worst case, this cannot be improved by a better choice of $C$, since for isotropic Gaussian data the integrand in \eqref{eq:exact_i} is linear in $u$ and varies by $L_t \|\widetilde{\mathbf{x}} - \mathbf{x}\|^2$ up to a constant factor. However, \citet{chen2026sampling} recover the rate of MALA for log-concave sampling using only gradients, so this extra factor is a cost of our estimator of the density ratio rather than of score-based access. The dependence on $\kappa$ is the price of removing the discretisation error by mixing. At noise levels close to the data, $p_t$ is as ill-conditioned as $p_{\text{data}}$, and the number of steps needed for a Langevin-type corrector to mix grows with $\kappa$, whether or not it is Metropolis-adjusted.

\section{Experiment details}
\label{appendix:experiment_details}
\paragraph{Synthetic 2D Datasets}
In these experiments, we used the ancestral sampler from DDPM \citep{ho2020ddpm} as the predictor. This uses the discretisation
\begin{align*}
    \mathbf{x}_{t} = \frac{1}{\sqrt{1-\beta_{t+1}}}(\mathbf{x}_{t+1} + \beta_{t+1}s_{\phi}(\mathbf{x}_{t+1}, t+1)) + \sqrt{\beta_{t+1}} \mathbf{z}_{i+1}
\end{align*}
where the forward noising process assumes a discrete Markov chain $\{\mathbf{x}_0, \dots, \mathbf{x}_{T}\}$ constructed as $p(\mathbf{x}_{t+1}| \mathbf{x}_{t}) \sim \mathcal{N}(\mathbf{x}_{t+1}; \sqrt{1-\beta_{t}}\mathbf{x}_{t}, \beta_{t}\mathbf{I})$ with $0<\beta_{1}, \dots \beta_{T}<1$. The diffusion sampling hyperparameters for these experiments are listed in Table \ref{table:hyperparameters_2d_datasets}.
\begin{table}[t]
\caption{Diffusion sampling hyperparameters for experiments on synthetic 2D datasets.}
\label{table:hyperparameters_2d_datasets}
\begin{center}
\begin{tabular}{lccc}
\toprule
Dataset & Predictor Steps & MADM steps per Predictor Step & Step Size $h$ at $p_t$\\
\midrule
spiral & 40 & 20 & $10^{-1} \times \beta_{t+1}$ \\
funnel & 10 & 20 & $\beta_{t+1}$\\
sierpinski triangle & 20 & 30 & $10^{-2} \times \beta_{t+1}$\\
pinwheel & 20 & 30 & $10^{-2} \times \beta_{t+1}$\\
\bottomrule
\end{tabular}
\end{center}
\end{table}

\paragraph{Image Datasets}
All experiments were run on a single NVIDIA H100 GPU. In these experiments, we used the probability flow ODE as the predictor by using pre-trained models (config F) from \citep{karras2022edm}. We used the same number of predictor steps (18 for CIFAR-10, 40 for FFHQ and AFHQv2, 256 for ImageNet). We use the default sampling strategy in their codebase, which uses the scale function $s(t)=1$ and schedule $\sigma(t) = t$. This yields a simplified expression for the probability flow ODE (Equation 4 in \citep{karras2022edm}) 
\begin{align*}
    \rmd \mathbf{Y}_{t} = \sigma (t_{\text{max}}-t) \nabla_{\mathbf{x}} \log p_{t_{\text{max}}-t}(\mathbf{Y}_{t})\rmd t.
\end{align*}
where the ODE runs forward in time (i.e. $t$ goes from $0$ to $t_{\text{max}}$). The diffusion sampling hyperparameters for these experiments are listed in Table \ref{table:hyperparameters_image_datasets}.
\begin{table}[t]
\caption{Diffusion sampling hyperparameters for experiments on image datasets. Run times are based on a batch size of 64.}
\label{table:hyperparameters_image_datasets}
\begin{center}
\begin{tabular}{lccccc}
\toprule
Dataset & ODE Solver& Predictor Steps  & MADM Steps & MADM Step Size $h$ at $p_t$ & Run Time\\
\midrule
CIFAR-10 & Heun & 18 & - & - & 3 secs\\
CIFAR-10 & Heun & 18 & 1 & $10^{-2} \times \sigma(t)$ & 3 secs\\
CIFAR-10 & Euler & 18 & - & - & 2 secs\\
CIFAR-10 & Euler & 18 & 1 & $10^{-2} \times \sigma(t)$ & 3 secs\\
FFHQ & Heun & 40 & - & - & 10 secs\\
FFHQ & Heun & 40 & 1 & $10^{-2}  \times\sigma(t)$ & 24 secs\\
FFHQ & Euler & 40 & - & - & 8 secs\\
FFHQ & Euler & 40 & 1 & $10^{-2} \times\sigma(t)$ & 20 secs\\
AFHQv2 & Heun & 40 & - & - & 10 secs\\
AFHQv2 & Heun & 40 & 1 & $10^{-2}  \times \sigma(t)$ & 24 secs\\
AFHQv2 & Euler & 40 & - & - & 8 secs\\
AFHQv2 & Euler & 40 & 1 & $10^{-2} \times \sigma(t)$ & 20 secs\\
ImageNet & Heun & 256 & - & - & 3 mins \\
ImageNet & Heun & 256 & 1 & $10^{-2}  \times \sigma(t)$ & 4 mins \\
ImageNet & Euler & 256 & - & - & 2 mins \\
ImageNet & Euler & 256 & 1 & $10^{-2} \times \sigma(t)$ & 3 mins \\
\bottomrule
\end{tabular}
\end{center}
\end{table}

\section{Additional diagnostics on CIFAR-10}

\subsection{Conservativeness of learned score}
Let $J_s=\nabla_xs_\theta$ and $K_s=(J_s-J_s^\top)/2$. We report
\begin{align*}
    R_{\rm skew}=100\frac{\|K_s\|_F^2}{\|J_s\|_F^2}
=50\frac{\sum_{i<j}(\partial_i s_{\theta,j}-\partial_j s_{\theta,i})^2}{\|J_s\|_F^2}.
\end{align*}
This is the percentage of Jacobian energy in the non-conservative component. A conservative score has $R_{\rm skew}=0$. We report this in Table \ref{table:score_conservativeness_diagnostics}. The score is therefore approximately conservative over the probed states: the skew component accounts for at most $0.66\%$ of total Jacobian energy.

\begin{table}[H]
\caption{Percentage of Jacobian energy in the non-conservative component. We estimate this without constructing the Jacobian, by using two Rademacher Jacobian-product probes at each of 16 states per noise level.}
\label{table:score_conservativeness_diagnostics}
\begin{center}
\begin{tabular}{lcccccc}
\hline
$\sigma$ & 57.6 & 5.32 & 0.585 & 0.0599 & 0.00753 & 0.002 \\
$R_{\rm skew}$ & $1.7{\times}10^{-6}\%$ & $0.0030\%$ & $0.109\%$ & $0.214\%$ & $0.664\%$ & $0.398\%$ \\
\bottomrule
\end{tabular}
\end{center}
\end{table}

\subsection{MADM robustness across proposal scales}
Many efforts have sought to integrate MCMC-based corrections into continuous diffusion models without Metropolis correction, for instance more broadly see \cite{Holl2026DiffusionRestore}. 
However, dismissing discretisation error overlooks its severe vulnerability under scaling. 
Comparing the Unadjusted Langevin Algorithm (ULA) and Metropolis-adjusted approaches like MADM across moderate and aggressive proposal scales illustrates this exact failure: while both methods perform similarly at moderate scales, a fourfold increase in the common scale causes ULA's performance to deteriorate sharply (from 1.97 to 7.27 in FID). 
Conversely, accept–reject adjustments (such as Algorithm \ref{algorithm:combined_accept_reject}) maintain robustness---retaining an FID of 1.99---by actively rejecting harmful proposals. 
This empirical contrast underscores the danger of glossing over discretisation bias, demonstrating that proper correction mechanisms are indispensable for maintaining structural stability and preventing severe performance collapse.
\begin{table}[H]
\caption{Fr\'echet Inception Distance (FID) on CIFAR-10 sampled using the pre-trained model from \cite{karras2022edm}. We consider the probability flow ODE (PF-ODE) predictor equipped with unadjusted Langevin algorithm (ULA) and Metropolis-Adjusted Diffusion Model (MADM) correctors. All FIDs are calculated using 50,000 samples. We highlight the best FID in bold.}
\label{table:madm_robustness_proposal_scales}
\begin{center}
\begin{tabular}{lccc}
\hline
PC Sampler/Proposal scale & $2{\times}10^{-2}$ & $8{\times}10^{-2}$ & $16{\times}10^{-2}$ \\
\hline
PF-ODE + ULA & 1.97 & 7.27 & 27.03\\
PF-ODE + MADM (Alg. \ref{algorithm:combined_accept_reject}) & 1.97 & \textbf{1.99} & \textbf{1.99} \\
\bottomrule
\end{tabular}
\end{center}
\end{table}

\section{Licenses}
\label{appendix:licenses}
Codebases:
\begin{itemize}
    \item Elucidating the Design Space of Diffusion-Based Generative Models \citep{karras2022edm}: Attribution-NonCommercial-ShareAlike 4.0 International.
\end{itemize}

Datasets:
\begin{itemize}
    \item CIFAR-10 \citep{krizhevsky2009learning}: MIT license.

    \item FFHQ \citep{karras2019gans}:  Creative Commons BY-NC-SA 4.0 license.

    \item AFHQv2 \citep{choi2026stargan}: MIT license.

    \item ImageNet \citep{deng2009imagenet}: Unknown license.
\end{itemize}


\newpage
\section*{NeurIPS Paper Checklist}

\begin{enumerate}

\item {\bf Claims}
    \item[] Question: Do the main claims made in the abstract and introduction accurately reflect the paper's contributions and scope?
    \item[] Answer: \answerYes{} 
    \item[] Justification: Our main theoretical and experimental contributions are clearly stated in the abstract and demonstrated in the paper. They reflect the paper’s contributions and scope.

\item {\bf Limitations}
    \item[] Question: Does the paper discuss the limitations of the work performed by the authors?
    \item[] Answer: \answerYes{} 
    \item[] Justification: We discuss the limitations of our work in Section \ref{sec:conclusion}.

\item {\bf Theory assumptions and proofs}
    \item[] Question: For each theoretical result, does the paper provide the full set of assumptions and a complete (and correct) proof?
    \item[] Answer: \answerYes{} 
    \item[] Justification: All proofs are provided in the technical appendices and appropriately referenced.

    \item {\bf Experimental result reproducibility}
    \item[] Question: Does the paper fully disclose all the information needed to reproduce the main experimental results of the paper to the extent that it affects the main claims and/or conclusions of the paper (regardless of whether the code and data are provided or not)?
    \item[] Answer: \answerYes{} 
    \item[] Justification: We provide full experiment details in Appendix \ref{appendix:experiment_details} and the code to run our experiments.

\item {\bf Open access to data and code}
    \item[] Question: Does the paper provide open access to the data and code, with sufficient instructions to faithfully reproduce the main experimental results, as described in supplemental material?
    \item[] Answer: \answerYes{} 
    \item[] Justification: We provide the code needed to run our experiments.

\item {\bf Experimental setting/details}
    \item[] Question: Does the paper specify all the training and test details (e.g., data splits, hyperparameters, how they were chosen, type of optimizer) necessary to understand the results?
    \item[] Answer: \answerYes{} 
    \item[] Justification: We provide full experiment details in Appendix \ref{appendix:experiment_details}.

\item {\bf Experiment statistical significance}
    \item[] Question: Does the paper report error bars suitably and correctly defined or other appropriate information about the statistical significance of the experiments?
    \item[] Answer: \answerNo{} 
    \item[] Justification: Our main metric is FID score for which it is standard practice to report it calculated on the first 50,000 images generated from the model. Standard deviations could be bootstrapped from this set but this is not standard practice. Furthermore, the computational cost to sample large image models 50,000 times multiple times is prohibitive for our academic compute cluster.

\item {\bf Experiments compute resources}
    \item[] Question: For each experiment, does the paper provide sufficient information on the computer resources (type of compute workers, memory, time of execution) needed to reproduce the experiments?
    \item[] Answer: \answerYes{} 
    \item[] Justification: We provide full details of compute resources in Appendix \ref{appendix:experiment_details}.
    
\item {\bf Code of ethics}
    \item[] Question: Does the research conducted in the paper conform, in every respect, with the NeurIPS Code of Ethics \url{https://neurips.cc/public/EthicsGuidelines}?
    \item[] Answer: \answerYes{} 
    \item[] Justification: After careful review of the NeurIPS Code of Ethics, it is clear that the research conducted in this paper conforms with the Code of Ethics in every respect.

\item {\bf Broader impacts}
    \item[] Question: Does the paper discuss both potential positive societal impacts and negative societal impacts of the work performed?
    \item[] Answer: \answerNA{} 
    \item[] Justification: This paper is mostly theoretical and methodological. We do not see immediate societal impact of our work. However, we acknowledge that our algorithm might suffer from the same societal biases as any other generative models. Indeed, it could improve the quality of generative models and hence, be used to generate deepfakes for disinformation.
    
\item {\bf Safeguards}
    \item[] Question: Does the paper describe safeguards that have been put in place for responsible release of data or models that have a high risk for misuse (e.g., pre-trained language models, image generators, or scraped datasets)?
    \item[] Answer: \answerNA{} 
    \item[] Justification: The paper poses no such risks.

\item {\bf Licenses for existing assets}
    \item[] Question: Are the creators or original owners of assets (e.g., code, data, models), used in the paper, properly credited and are the license and terms of use explicitly mentioned and properly respected?
    \item[] Answer: \answerYes{} 
    \item[] Justification: All datasets used in our experiments are open-source and have been properly referenced in the paper and Appendix \ref{appendix:licenses}.

\item {\bf New assets}
    \item[] Question: Are new assets introduced in the paper well documented and is the documentation provided alongside the assets?
    \item[] Answer: \answerNA{} 
    \item[] Justification: The paper does not release new assets.

\item {\bf Crowdsourcing and research with human subjects}
    \item[] Question: For crowdsourcing experiments and research with human subjects, does the paper include the full text of instructions given to participants and screenshots, if applicable, as well as details about compensation (if any)? 
    \item[] Answer: \answerNA{} 
    \item[] Justification: The paper does not involve crowdsourcing nor research with human subjects.

\item {\bf Institutional review board (IRB) approvals or equivalent for research with human subjects}
    \item[] Question: Does the paper describe potential risks incurred by study participants, whether such risks were disclosed to the subjects, and whether Institutional Review Board (IRB) approvals (or an equivalent approval/review based on the requirements of your country or institution) were obtained?
    \item[] Answer: \answerNA{} 
    \item[] Justification: The paper does not involve crowdsourcing nor research with human subjects.

\item {\bf Declaration of LLM usage}
    \item[] Question: Does the paper describe the usage of LLMs if it is an important, original, or non-standard component of the core methods in this research? Note that if the LLM is used only for writing, editing, or formatting purposes and does \emph{not} impact the core methodology, scientific rigor, or originality of the research, declaration is not required.
    \item[] Answer: \answerNA{} 
    \item[] Justification: The core method development in this research does not involve LLMs as any important, original, or non-standard components. However, we acknowledge that LLMs were used to check the soundness of proofs and theorems.

\end{enumerate}

\end{document}